\documentclass[twoside]{article}

\usepackage[accepted]{aistats2025}
\usepackage[round]{natbib}

\usepackage[utf8]{inputenc} 
\usepackage[T1]{fontenc}    
\usepackage{hyperref}       
\usepackage{url}            
\usepackage{booktabs}       
\usepackage{amsfonts}       
\usepackage{nicefrac}       
\usepackage{microtype}      
\usepackage[dvipsnames]{xcolor}         

\usepackage{microtype}
\usepackage{graphicx}

\usepackage{amsmath}
\usepackage{amssymb}
\usepackage{mathtools}
\usepackage{amsthm}

\usepackage[capitalize,noabbrev]{cleveref}

\theoremstyle{plain}
\newtheorem{theorem}{Theorem}[section]

\newtheorem{lemma}[theorem]{Lemma}

\theoremstyle{definition}

\theoremstyle{remark}

\usepackage{multirow}
\usepackage{bm}
\usepackage{subcaption} 

\newcommand{\statespace}{\mathcal{S}}
\newcommand{\actionspace}{\mathcal{A}}
\newcommand{\obsfunc}{\mathcal{O}}
\newcommand{\obsspace}{\Omega}
\newcommand{\histspace}{\mathcal{H}}

\newcommand{\agentset}{\mathcal{N}}
\newcommand{\corrset}{\mathcal{C}}
\newcommand{\levelset}{\mathcal{G}}

\newcommand{\EE}{\mathbb{E}}
\newcommand{\RR}{\mathbb{R}}
\newcommand{\Var}{\mathrm{Var}}
\newcommand{\Cov}{\mathrm{Cov}}
\newcommand{\mixer}{f_\text{Mix}}

\newcommand{\stopgrad}{\text{sg}}

\usepackage{easyReview}
\usepackage[export]{adjustbox}
\usepackage{cancel}
\usepackage{svg}
\usepackage{wrapfig}

\begin{document}

%
\runningtitle{Multi-level Advantage Credit Assignment}

%

\twocolumn[

\aistatstitle{Multi-level Advantage Credit Assignment\\for Cooperative Multi-Agent Reinforcement Learning}

\aistatsauthor{ Xutong Zhao \And Yaqi Xie }

\aistatsaddress{ Mila - Quebec AI Institute\\~Polytechnique Montr\'eal \And School of Computer Science,\\Carnegie Mellon University } ]

\begin{abstract}
Cooperative multi-agent reinforcement learning (MARL) aims to coordinate multiple agents to achieve a common goal. A key challenge in MARL is credit assignment, which involves assessing each agent's contribution to the shared reward. Given the diversity of tasks, agents may perform different types of coordination, with rewards attributed to diverse and often overlapping agent subsets. In this work, we formalize the credit assignment level as the number of agents cooperating to obtain a reward, and address scenarios with multiple coexisting levels. We introduce a multi-level advantage formulation that performs explicit counterfactual reasoning to infer credits across distinct levels. Our method, Multi-level Advantage Credit Assignment (MACA), captures agent contributions at multiple levels by integrating advantage functions that reason about individual, joint, and correlated actions. Utilizing an attention-based framework, MACA identifies correlated agent relationships and constructs multi-level advantages to guide policy learning. Comprehensive experiments on challenging Starcraft v1\&v2 tasks demonstrate MACA's superior performance, underscoring its efficacy in complex credit assignment scenarios.

\end{abstract}

\section{INTRODUCTION}
\looseness-1
Multi-agent systems (MAS) are inherently applicable to a broad spectrum of real-world applications, ranging from smart grid \citep{roesch2020smart} to swarm robotics \citep{huttenrauch2017guided}, and autonomous vehicles \citep{shalev2016safe}.
Cooperative multi-agent reinforcement learning (MARL) emerges as a general learning framework for MAS with a common objective among agents.
Naively applying single-agent RL methods by regarding multiple agents as one single agent with exponentially large joint action space is limited by its scalability, and potentially leads to non-stationary learning.
A key challenge in MARL is multi-agent credit assignment, the task of identifying the contribution of each individual agent's action to the global reward \citep{albrecht2023multi}.
Consider a warehouse task where agents receive collective rewards for moving objects.
In this scenario, one agent \textsl{A} might collaborate with agents \textsl{B} and \textsl{C} to carry a fridge, while another agent, \textsl{D}, does nothing. 
Disentangling each agent's contribution to the total collective reward is crucial to promoting effective cooperation strategies.
However, credit assignment is highly non-trivial given only the joint actions and shared rewards.
Real-world scenarios often involve more complex multi-agent interactions, rendering credit assignment more challenging.


Addressing the multi-agent credit assignment challenge efficiently and scalably remains an open question.
Prior efforts to address it can be mainly categorized into: (1) implicit methods, such as QMix~\citep{rashid2018qmix}, which learns a mixing network to decompose joint values, and (2) explicit methods, such as COMA~\citep{foerster2018counterfactual}, which performs counterfactual reasoning to guide the learning of individual policies. 
However, both lines of works only consider cooperation among a certain fixed number of agents and overlook contributions from different or overlapping agent subsets.

\looseness-1
Inspired by the study of human collaboration in cognitive neuroscience~\citep{richerson2016cultural}, we propose a \textit{multi-level} formulation to model the credit assignment in multi-agent cooperation.
This approach considers the general case where each agent collaborates with distinct subsects of agents.
We formalize the type of cooperation into distinct \textit{levels}, where each level corresponds to the number of agents needed to obtain the reward (detailed in \cref{sec:ca-level}).
For example, in the warehouse scenario described above, agent \textsl{A} is involved in a $3$-level cooperation. Additionally, agent \textsl{A} may simultaneously carry a backpack, resulting in both $1$-level and $3$-level cooperations.
This coexistence of \textit{multi-level} cooperation underscores the necessity of our \textit{multi-level} credit assignment formulation, designed to explicitly recognize and assign credit across different levels of cooperation.
More specifically, we propose a $k$-level counterfactual credit assignment formulation, further extending it to a general multi-level credit assignment to accommodate coexisting levels.
Inspired by \citet{foerster2018counterfactual}, we introduce a $k$-level counterfactual advantage function.
The advantage is defined as the discrepancy between the state-action value of the taken joint action and a counterfactual baseline, which marginalizes out the $k$ agents' actions by their policies while keeping other agents' actions fixed.
This formulation \textit{explicitly} performs counterfactual reasoning that deduces the joint contribution by the $k$-agent subset.
Accounting for circumstances with coexisting different levels, we extend $k$-level credit assignment to multi-level credit assignment by amalgamation of multiple $k$-level advantages with different $k$'s.
Since each $k$-level advantage function is tailored to a specific level, the multi-level advantage captures contributions across different levels.
\looseness-1
We then introduce Multi-level Advantage Credit Assignment (MACA), an actor-critic based approach to address the multi-agent credit assginment challenge based on our \textit{multi-level} counterfactual credit assignment formulation.
MACA's multi-level advantage is constructed by three $k$-level advantages that respectively attribute contributions to individual actions, joint actions, and actions taken by the subset of strongly correlated agents.
Although tasks may involve a variety of credit assignment levels, MACA focuses on the most significant ones. 
It adapts the transformer encoder architecture \citep{vaswani2017attention}, leveraging its attention mechanism to model correlations among agents \citep{zhang2022relational}, which is further used to construct the multi-level counterfactual advantage.
The \textit{dynamically} determined agent correlation serves as a basis for adaptive credit assignment at different states, which promotes effective contribution estimation and efficient learning across various scenarios.

\looseness-1
We evaluate MACA on two popular cooperative MARL benchmarks, the StarCraft Multi-Agent Challenge (SMAC) \citep{samvelyan2019starcraft}, and the newly extended SMACv2 benchmark, which introduces significantly greater stochasticity ~\citep{ellis2022smacv2}.
Tasks in both benchmarks cover various credit assignment types, making them ideal for testing multi-agent credit assignments.
Our empirical findings highlight MACA's ability to enhance performance effectively and robustly, particularly in the more demanding contexts of the SMACv2 tasks. 
Additionally, ablation studies showcase that every single level is essential in MACA's multi-level advantage.
Theoretical analysis establishes an in-depth understanding of the strong performance.


Our main contributions are summarized as follows:
\begin{itemize}
    \item Inspired by human collaboration, we propose a multi-level credit assignment formulation that recognizes and assigns credit to individual agents across different levels of collaboration.
    \item Based on the formulation, we present MACA, an actor-critic approach that addresses the multi-agent credit assignment challenge by leveraging multi-level counterfactual advantage functions to capture contributions of individual actions, joint actions, and actions by strongly correlated agents.
    \item Experiments on challenging Starcraft benchmarks demonstrate MACA outperforms previous state-of-the-art. 
    Notably, it shows more significant advantages in tasks of higher complexity.
\end{itemize}

\section{RELATED WORK}


Recent works on multi-agent credit assignment can be categorized into two lines of approaches based on the utilized assignment mechanism:
(1) \textit{implicit methods}, which implicitly learns agent contribution via a decomposition of the joint value function, and
(2) \textit{explicit methods}, which rely on some explicit mechanism to infer agents' contributions.

\paragraph{Implicit Credit Assignment}
\looseness-1
Most prior works perform implicit credit assignment by learning a decomposition of the joint value function to per-agent value functions \citep{sunehag2017value,rashid2018qmix,son2019qtran,wang2020qplex}.
VDN \citep{sunehag2017value} assumes linearity and obtains the joint value by summing over all individual values.
QMix \citep{rashid2018qmix} lifts the linearity assumption by learning a mixing neural network, but its monotonicity constraints on the mixing weights also limit its expressiveness.
HAPPO \citep{kuba2021trust} is based on the multi-agent advantage decomposition and employs a sequential policy update scheme, but the computational complexity of sequential updates limits its scalability.

\paragraph{Explicit Credit Assignment}
\looseness-1
One exemplar family calculates \textit{difference rewards} against a reward baseline \citep{tumer2007distributed,wolpert2001optimal}.
COMA \citep{foerster2018counterfactual} is an actor-critic (AC) method implementing this idea through a counterfactual advantage baseline.
However, it treats each agent as an independent entity without considering interactions within the group, thus can be inefficient in learning complex cooperative behaviours \citep{papoudakis2020benchmarking,li2021shapley}.
Other AC methods such as MAPPO \citep{yu2021surprising} and MAA2C \citep{papoudakis2020benchmarking} share a similar formulation, but their advantage is only concerned with joint actions' contribution, without distinguishing individual agents' credits. 
Limited by their credit assignment mechanism, these AC methods only deduce contributions of the same set of agents, ignoring the potential situations where one agent may contribute through different or multiple overlapping agent subsets.

\looseness-1
Other representative methods leverage the Shapley Value \citep{shapley1953value},  but its computational complexity grows factorially with the number of agents \citep{kumar2020problems}, handicapping efficient training in practice.
\citet{li2021shapley} seeks to reduce computations by approximating Shapley Value through Monte Carlo sampling, but this formulation sacrifices certain properties of Shapley Value, such as symmetry.
Other explicit methods employ reward shaping.
\citet{li2022difference} applies potential-based difference rewards, but learning its reward model introduces more computational cost compared to model-free approaches.

\paragraph{Attention in MARL}
\looseness-1
Adopting transformers or attention mechanisms \citep{vaswani2017attention} to MARL has grown in popularity in recent years.
\citet{meng2021offline} extends the Decision Transformer \citep{chen2021decision} to the offline MARL regime.
\citet{baker2019emergent} uses an attention-based architecture to learn permutation-invariant state representation and demonstrates self-supervised autocurriculum in the hide-and-seek task.
\citet{seraj2022learning} constructs a heterogeneous graph-attention network architecture to promote efficient multi-agent communication.
\citet{nayak2023scalable} also use the attention mechanism for inter-agent communication, seeking to tackle the 2D navigation problem.
However, few works have leveraged attention-captured correlations to address the credit assignment challenge in cooperative MARL.



\section{PRELIMINARIES}\label{sec:preliminary}
\subsection{Dec-POMDP}
We consider the cooperative multi-agent setting as a decentralized partially observable Markov decision process (Dec-POMDP) \citep{oliehoek2016concise}, which is formally defined as a tuple $G = \langle \agentset, \statespace, \actionspace, P, R, \gamma, \obsfunc, \obsspace \rangle$
, where $\agentset = \{1, \dots, n\}$ is the set of agents, $\statespace$ is the state space, $\actionspace = \actionspace_1 \times \cdots \times \actionspace_n$ is the joint action space, $P: \statespace \times \actionspace \times \statespace \rightarrow [0,1]$ is the transition probability, $R: \statespace \times \actionspace \times \statespace \rightarrow \RR$ is the global reward function, $\gamma \in [0,1]$ is the discount factor, $\obsspace = \obsspace_1 \times \cdots \times \obsspace_n$ is the joint observation space, $\obsfunc: \statespace \times \agentset \rightarrow \obsspace$ is the observation function.
At each timestep $t$, agents are in a state $s^t \in \statespace$.
Each agent $i \in \agentset$ observes a local observation $o^t_i = \obsfunc(s^t, i) \in \obsspace_i$, and selects an action $a^t_i \in \actionspace_i$.
All actions together form the joint action $a^t=[a^t_i]^n_{i=1}$.
The environment transitions to the next state $s^{t+1} \sim P(s'|s^t, a^t)$, and outputs a global reward $r \in R$ shared across all agents.

\looseness-1
Let $\histspace_i = (\obsspace_i \times \actionspace_i)^* \times \obsspace_i$ represent the space of agent $i$'s action-observation history.
A trajectory is a full joint history sequence $\tau = \langle o^t, a^t \rangle^\infty_{t=0}$.
We use $\pi = \pi_1 \times \cdots \times \pi_n$ to denote the joint policy, where each $\pi_i(a'_i|h^t_i): \histspace_i \times \actionspace_i \rightarrow [0,1]$ is agent $i$'s policy distribution that samples its action at every timestep.
The marginal state distribution $d^\pi$ is induced by the policy $\pi$ and the transition probability $P$.
For the simplicity of notations, in subsequent sections we assume the Markov property and full observability, and we may ignore superscripts/subscripts if the context is unambiguous.
We use $-i$ to denote all agents except agent $i$.

\looseness-1
The discounted return is defined as $G^t = \sum^\infty_{k=0} \gamma^k r^{t+k}$.
Based on agents' joint policy, the state-value function and action-value function are defined as $V_\pi(s) = \EE_{s^{1:\infty} \sim d^\pi,a^{0:\infty} \sim \pi} \left[ G^0 | s^0=s \right]$ and $Q_\pi(s, a) = \EE_{s^{1:\infty} \sim d^\pi,a^{1:\infty} \sim \pi} \left[ G^0 | s^0=s,a^0=a \right]$, respectively.
By definition $V_\pi(s) = \EE_{a \sim \pi} \left[ Q_\pi(s, a) \right]$.
The advantage function is defined as $A_\pi(s, a) = Q_\pi(s, a) - V_\pi(s)$.
The MARL objective is to find a joint policy that maximizes the expected return $J(\theta) = \EE_\tau \left[ G^0 \right]$.

\subsection{Policy Gradient Methods}
\label{sec:pg-methods}
In RL, policy gradient (PG) methods optimize the policy $\pi$ by performing gradient updates on the objective $J(\theta)$, where $\pi$ is parameterized by $\theta$.
The Policy Gradient Theorem \citep{sutton1999policy} provides the gradient with respect to $\theta$ as $\nabla_\theta J(\theta) = \EE_{s \sim d^\pi, a \sim \pi} \left[ Q(s, a) \nabla_\theta \log \pi_\theta(a|s) \right]$.
The MAPG theorem \citep{foerster2018counterfactual,zhang2018fully} extends the single-agent policy gradient to
\begin{align}\label{eq:ma-pg}
    g_{\theta_i} &= \nabla_{\theta_i} J(\theta) = \EE_{s \sim d^\pi, a \sim \pi} \left[ Q(s, a) \nabla_{\theta_i} \log \pi_i(a_i|s) \right]
\end{align}

\looseness-1
In the AC framework, a critic learns the value function, constructing the gradient to update the actor, i.e. the policy.
A popular variance-reduction approach is to subtract the $Q$-value estimate by a baseline function $b$ \citep{weaver2013optimal}.
One common choice $b(s) = V(s)$ gives the advantage function $Q(s,a)-V(s)$ by definition.
We can show that any action-independent baseline function preserves unbiasedness of the gradient estimate in expectation, i.e., $\EE_{s \sim d^\pi, a \sim \pi} \left[ b_i \nabla_{\theta_i} \log \pi_i(a_i|s) \right] =0$
for any agent $i \in \agentset$ if $b_i$ does not depend on $a_i$ (see 
\cref{app:theoretical-results} 
for proof).

In this work, we follow the centralized training with decentralized execution (CTDE) \citep{bernstein2002complexity} MARL learning paradigm.
Extending AC approaches to MARL, the CTDE formulation is usually centralized critic and decentralized actors, where the critic learns the value function based on the global state information (and potentially joint actions), which is used to optimize the policies following \cref{eq:ma-pg}.

\subsection{Credit Assignment Level}\label{sec:ca-level}
The concept of credit assignment level is inspired by the coordination level in MARL tasks \citep{liu2022stateful}.
Let $\levelset \subset \agentset$ denote a subset of $|\levelset| = k$ agents.
Different subsets are allowed to overlap.
At each timestep $t$, we use $r_\levelset(s^t, a^t) = r_\levelset(s^t, a^t_\levelset)$ to denote the joint reward that can only be obtained if the subset of $k$ agents cooperate, where $a^t_\levelset = \{a^t_i | i \in \levelset \}$.
We then represent the global reward as the sum of rewards contributed by all subsets of agents:
$r(s^t, a^t) = \sum_{\levelset \subset \agentset} r_\levelset(s^t, a^t)$.
We consider the number of agents $k$ as a credit assignment level.
We then consider credit assignment and associated contributions with different levels as different types.
This formalism intends to cover complex situations that involve an individual agent in different credit assignment types across timesteps, or even coexisting different types at the same timestep.


\section{METHOD}\label{sec:method}
\looseness-1
In this section, we introduce MACA, an actor-critic MARL method that \textit{explicitly} models multiple levels of credit assignment. 
This approach delineates three distinct advantage functions, each designed to assess contributions from individual actions, joint actions, and actions taken by strongly correlated partners. 
With multi-level contributions encoded by the combination of these advantages, MACA provides a comprehensive framework for addressing multi-agent credit assignment challenges, accommodating the diverse range of interactive behaviors among agents. 


\subsection{Multi-Level Counterfactual Formulation}\label{sec:maca-adv-baseline}

\paragraph{$k$-Level Advantage}
Motivated by \citet{foerster2018counterfactual},
we consider the advantage in the MAPG estimator as an approach to counterfactually assessing agents' contributions.
With a centralized critic learning $Q(s,a)$, when we compute the advantage function $A_i(s,a)$ for each agent $i$, the actions marginalized out in the $Q$-based baseline correspond to the agents whose contributions we reason about.
In particular we can write a $k$-level counterfactual baseline 
\begin{align}\label{eq:k-level-baseline}
    b^\textsc{CF}_i(s,a) = \EE_{a_{\levelset_i}} \left[ Q(s, a) \right]
\end{align}
where $a_{\levelset_i} = \{a_j \sim \pi_j: j \in \levelset_i \}$ and $\levelset_i$ is an arbitrary subset of $|\levelset_i|=k$ agents that satisfies $\{i\} \subset \levelset_i \subset \agentset$.
The resulting advantage $A_i(s,a)=Q(s,a)-b^\textsc{CF}_i(s,a)$ provides a general form of counterfactual advantage that explicitly reasons about credit assignment by the $k$-agent subset.
For instance, 
COMA computes the baseline $b^\textsc{Ind}_i(s, a) = \EE_{a_i \sim \pi_i}[Q(s, a)]$,
thereby reasoning about the \textit{individual action}'s contribution.
MAPPO/MAA2C estimates the baseline $b^\textsc{Jnt}_i(s, a) = V(s) = \EE_{a \sim \pi}[Q(s,a)]$.
Intuitively, this advantage evaluates how much the \textit{joint action} is better or worse than the default behaviour.
As all actions are marginalized out, they do not differentiate each particular agent's contribution.
It is trivial to show that the $b^\textsc{CF}_i$ encompasses both $b^\textsc{Ind}_i$ and $b^\textsc{Jnt}_i$.

\looseness-1
While the minimum-variance baseline $b^*_i$ (see 
\cref{thm:min-var-baseline}
) is usually infeasible to compute in complex MDPs, our $k$-level counterfactual baseline $b^\textsc{CF}_i$ also reduces the variance of MAPG estimates and assists learning in practice.
The rational is clear when each $k$-level baseline is expressed as $b^\textsc{CF}_i = b^*_i - \frac{\Cov \left( Q, || \nabla_i \log \pi_i ||^2 \right)}{\EE \left[ || \nabla_i \log \pi_i ||^2 \right]}$ (see 
\cref{thm:min-var-vs-value}
).
The term $\Cov \left( Q, || \nabla_i \log \pi_i ||^2 \right)$ typically becomes negative as the gradient becomes smaller on actions with high returns during the optimization process, leading to the baseline being an optimistic baseline \citep{chung2021beyond},
which justifies our design choice.

We consider one important subset $\levelset_i$ whose agents are strongly correlated with agent $i$.
We name this subset the \textit{CorrSet} $\corrset_i$, and the corresponding baseline the \textit{CorrSet} baseline $b^\textsc{Cor}_i$.
We evaluate inter-agent correlations conditioned on the state information so that $b^\textsc{Cor}_i$ is capable of adapting dynamically to different credit assignment levels across timesteps.
We discuss how we determine $\corrset_i$ in detail in the following subsection.

\begin{figure*}[t]
\centering
\includegraphics[width=0.99\linewidth]{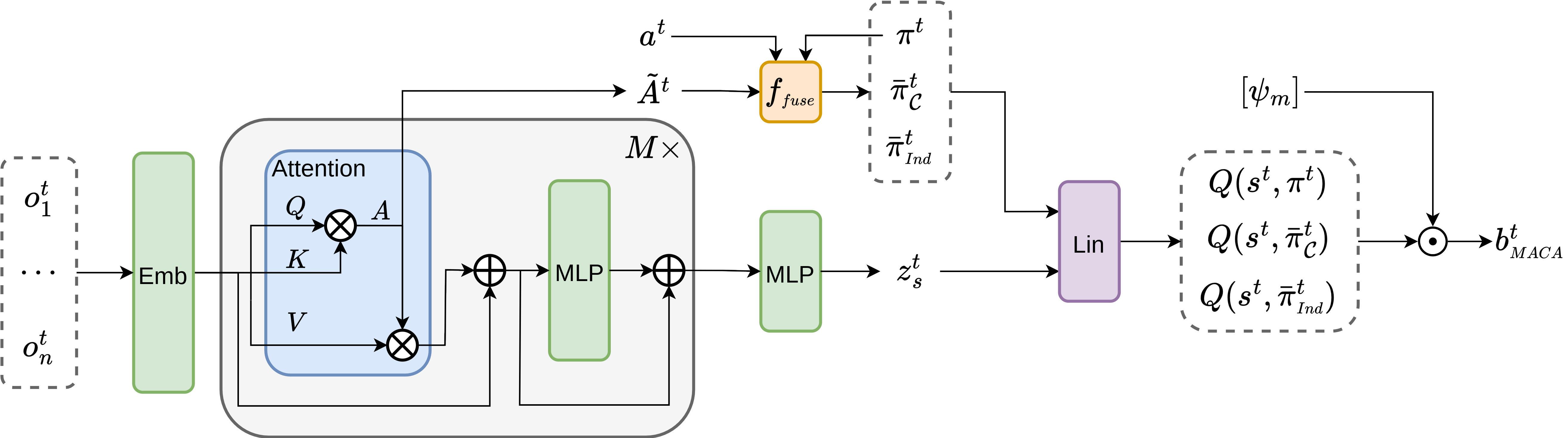}
\caption{MACA critic architecture. The sequence of agent observations $(o^t_1,\dots,o^t_n)$ inputs to an embedding layer and a self-attention encoder, and the output passes through an MLP layer to produce the state embedding $z_s$. The attention weight matrix $\Tilde{A}^t$, joint actions $a^t$, and joint policy distribution $\pi^t$ pass through the function $f_{fuse}$ to obtain action distributions $\bar{\pi}^t_{\corrset},\bar{\pi}^t_{Ind}$, corresponding to CorrSet and individual actions, respectively. The embedding $z_s$ and action distributions are fed to a linear layer to get respective Q values, and the coefficients $[\psi_m]$ weight them to obtain the final MACA baseline $b^t_{MACA}$.}
\label{fig:maca-architecture}
\end{figure*}

\paragraph{Multi-Level Advantage}
\looseness = -1
As discussed in previous sections, we consider the situation where one agent's contribution may simultaneously involve a mixture of \textit{multiple} credit assignment levels.
We aim to reason about each of them using a different $k$-level advantage function, and combine them into one advantage, which we call the \textit{multi-level} counterfactual advantage.
To obtain such an advantage, due to the same action value function used in all advantage functions, equivalently we only need to compute corresponding $k$-level baselines and a combination of them.
In particular, for each agent $i$, we compute the three most important $k$-level baselines: (1) $k=n$: the joint action set baseline $b^\textsc{Jnt}$, (2) $k=1$: the individual action set baseline $b^\textsc{Ind}_i$, and 
(3) $k=|\corrset_i|$: a \textit{CorrSet} baseline $b^\textsc{Cor}_i$.
In principle, a multi-level advantage does not limit the number of $k$-level advantage components.

\looseness = -1
We then obtain our MACA baseline by integrating these three baselines by a weighted sum
\begin{align}
    b^\textsc{MACA}_i &= \psi^\textsc{Jnt}_i b^\textsc{Jnt} + \psi^\textsc{Ind}_i b^\textsc{Ind} + \psi^\textsc{Cor}_i b^\textsc{Cor} \label{eq:maca-baseline} \\
    A^\textsc{MACA}_i &= Q(s,a)-b^\textsc{MACA}_i \label{eq:maca-advantage}
\end{align}
where the weighting coefficients $[\psi^m]_{m \in [3]} \in \Delta(3)$ are state-dependent.
Hence the resulting advantage function $A^\textsc{MACA}_i$ reasons different credit assignment levels via its multi-level baselines.
We discuss how we learn the coefficients using stochastic optimization in the following subsection.
As each of the $k$-level baselines marginalizes out the action $a_i$, the resulting MACA baseline is also independent of $a_i$, which preserves unbiasedness in policy gradient estimates (see 
\cref{thm:baseline-unbiasedness}
).
This property ensures MACA advantage is generally compatible with the MAPG framework and applicable to various algorithms.
\cref{thm:suboptimality} 
derives the convergence of MACA to a local optimal policy, by following the convergence proof of single-agent actor-critic \citep{sutton1999policy,konda1999actor} subjecting to the same assumptions.
We hereby provide a proof sketch. The detailed proof is deferred to 
\cref{app:theoretical-results}.

\looseness-1
\textit{Proof sketch of 
\cref{thm:suboptimality}.
}
Since each $k$-level baseline marginalizes out the action $a_i$, the multi-level baseline as a linear combination of all $k$-level baselines is independent of $a_i$.
By 
\cref{thm:baseline-unbiasedness} 
it does not introduce bias to MAPG estimates. 
Therefore the multi-level baseline does not affect the convergence in expectation. 
Writing the joint policy as the product of individual policies, the remaining proof directly follows \citet{konda1999actor}.

\begin{table*}[t] 
\renewcommand{\arraystretch}{1.5} 
\centering
\small
\caption{Overview of advantage functions.}
\label{tab:algo-overview}
\begin{tabular}{*{3}{l}}
\toprule
\multicolumn{1}{l}{Algorithm} & \multicolumn{1}{l}{Update} & \multicolumn{1}{l}{Advantage function $A_i(s, a)$} \\
\midrule
COMA & Simultaneous & $Q(s, a) - \EE_{a_i \sim \pi_i} \left[ Q(s, a) \right]$ \\
MAPPO & Simultaneous & $Q(s, a) - \EE_{a \sim \pi} \left[ Q(s, a) \right]$ \\
IPPO & Simultaneous & $Q_i(s, a_i) - \EE_{a_i \sim \pi_i} \left[ Q_i(s, a_i) \right]$ \\
PPO-Sum & Simultaneous & $Q_i(s, a_i) - \EE_{a_i \sim \pi_i} \left[ Q_i(s, a_i) \right]; Q = \sum_i Q_i$ \\
PPO-Mix & Simultaneous & $Q_i(s, a_i) - \EE_{a_i \sim \pi_i} \left[ Q_i(s, a_i) \right]; Q = \mixer([Q_i]), \frac{\partial Q}{\partial Q_i} \geq 0$ \\
HAPPO & Sequential & $( \prod_{j=1}^{i-1} \frac{\bar{\pi}_j(a_j|s)}{\pi_j(a_j|s)} ) \left( Q(s, a) - \EE_{a \sim \pi} \left[ Q(s, a) \right] \right)$ \\
\midrule
MACA-\textsc{Cor} (ours) & Simultaneous & $Q(s, a) - \EE_{a_{\corrset_i}} \left[ Q(s, a) \right]; \{i\} \subset \corrset_i \subset \agentset$ \\
\multirow{2}{*}{MACA (ours)} & \multirow{2}{*}{Simultaneous} & $Q(s, a) - \left(\psi^1 \EE_{a \sim \pi} \left[ Q(s, a) \right] + \psi^2 \EE_{a_i \sim \pi_i} \left[ Q(s, a) \right] \right. $ \\
& & $\left. + \psi^3 \EE_{a_{\corrset_i}} \left[ Q(s, a) \right] \right); \{i\} \subset \corrset_i \subset \agentset, [\psi^m] \in \Delta(3)$ \\
\bottomrule
\end{tabular}
\end{table*}

\subsection{Attention-based Framework}\label{sec:critic-architecture}
In this section, we introduce our learning framework that constructs the MACA advantage function.
We discuss how we utilize the self-attention mechanism to quantify agent-wise relative correlation and construct the \textit{CorrSet}.
We then present how we perform optimization to learn policies, values, and parameters that compute the MACA baseline.
The MACA architecture is illustrated in \cref{fig:maca-architecture}.



\paragraph{Critic Encoder Architecture}
\looseness-1
The critic leverages the multi-agent transformer encoder \citep{vaswani2017attention,wen2022multi}.
It inputs the sequence of agents' observations $(o^t_1, \dots, o^t_n)$ to an embedding layer and $M$ encoding blocks.
Each block consists of a self-attention mechanism, a multi-layer perceptron (MLP) layer, and residual connections and layer normalizations \citep{ba2016layer} to avoid gradient vanishing.
The output passes through an MLP layer to produce a state embedding $z_s$.
We compute the attention rollout weights \citep{abnar2020quantifying} of the last layer, denoted as $\Tilde{A}$ (referred to as attention weights for simplicity).

\paragraph{Value Estimation}
We aim to compute any $k$-level counterfactual baseline $b^\textsc{CF}_i(s,a)$.
To improve scalability and flexibility, we adopt a design that inputs the joint \textit{action distribution parameters} \citep{wierstra2007policy} by $b^\textsc{CF}_i(s,a) = \EE_{a_{\levelset_i}} \left[ Q(s, a) \right] = Q_{\phi}(s, \bar{\pi}_{\levelset_i})$,
where $\bar{\pi}_{\levelset_i} = \EE_{a_{\levelset_i}}[a] \in \RR^{|\actionspace|}$ is the action/policy distribution, and $Q_{\phi}$ is the value approximation parameterized by $\phi$. 
Regarding per-agent policies in $\bar{\pi}_{\levelset_i}$, the marginalized actions $\{a_j: j \in \levelset_i\}$ correspond to the original policies, and all other policies are one-hot with the probability of taken actions equal to 1.
We overabuse the same notation to denote the policy distribution whose joint action and individual action are marginalized out by $\bar{\pi}^\textsc{Jnt}_i$ and $\bar{\pi}^\textsc{Ind}_i$, respectively.
Note that $\bar{\pi}^\textsc{Jnt}_i = \pi$, the original joint policy.

\looseness-1
It is worth noting that in general $\EE_{a_{\levelset_i}}[Q(s, a)] \neq Q(s, \bar{\pi}_{\levelset_i})$.
While recent work has ignored this issue \citep{zhou2020learning}, we seek a sound solution that ensures equality. 
As per Jensen's inequality \citep{jensen1906fonctions} we feed the state embedding $z_s$ and $\bar{\pi}_{\levelset_i}$ through a \textit{linear} layer $b^\textsc{CF}_i(s,a) = \textrm{Lin}(z_s, \bar{\pi}_{\levelset_i})$.
To validate this layer's expressivity, we investigate a transformer-decoder modeling option in ablation experiments in \cref{sec:experiment-ablation}.

\paragraph{\textit{CorrSet} Construction}
We utilize inter-agent relationships represented by the pre-computed attention weights $\Tilde{A}$ to infer the \textit{CorrSet} $\corrset_i$.
Note that this component only relies on linear computations.

For each agent $i$, $\Tilde{A} \in \RR^n$ contains weights associated with all $n$ agents, each of which is denoted by $\Tilde{A}_{i,j}, j \in \agentset$.
The self-attention mechanism reasons correlations among input tokens \citep{zhang2022relational}. 
Intuitively, higher attention weights indicate stronger correlations between input tokens, and vice versa.
Hence to identify the subset of strongly correlated agents, we find agents with high attention weights -- i.e., 
let $j \in \corrset_i$ if $\Tilde{A}_{i,j} \geq \sigma$, where $\sigma \in [0,1]$ is a thresholding hyperparameter. 
To preserve unbiasedness in gradient estimates, we additionally enforce $i \in \corrset_i$ regardless of $\Tilde{A}_{i,i}$ as per definition of $k$-level baseline.
With the \textit{CorrSet} $\corrset_i$ settled, 
we can follow the procedure described above to compute all baselines ($b^\textsc{Jnt},b^\textsc{Ind}_i,b^\textsc{Cor}_i$) and further the MACA advantage. 

\paragraph{Training Losses}\label{para:training}
We follow the standard actor-critic training paradigm.
We optimize the actors through the MAPG as in \cref{eq:ma-pg}.
We update value function parameters $\phi$ on-policy by optimizing the mean-squared TD error over collected trajectories:
\begin{align*}
    L_{V}(\phi) &= \EE_\tau \left[ || V(s^t) - \stopgrad\left( r^t + \gamma V(s^{t+1}) \right) ||^2_2 \right] \\ 
    L_{Q}(\phi) &= \EE_\tau \left[ || Q(s^t, a^t) - \stopgrad\left( r^t + \gamma Q(s^{t+1}, a^{t+1}) \right) ||^2_2 \right] 
\end{align*}
where $V(s^t) = Q(s^t, \pi)$ as discussed above, and $\stopgrad(\cdot)$ indicates the stop-gradient operation.

\looseness-1
To learn the weighting coefficients $[\psi^\textsc{Jnt}_i,\psi^\textsc{Ind}_i,\psi^\textsc{Cor}_i]$ involved in the computation of MACA advantage (\cref{eq:maca-advantage}), it is important to realize they indirectly affect performance improvement.
Since policy gradient updates only rely on the value of advantage functions, the MAPG objective is indifferentiable with respect to coefficients $\psi$'s.
Similarly the TD updates performed on value functions do not optimize the weights either.
We hence consider the stochastic optimization setting.
We compute each coefficient by a linear layer $Lin(z_s; \eta)$ parameterized by $\eta$, and apply a softmax over all coefficients to obtain a probability distribution.
We leverage the CMA-ES method \citep{hansen2019pycma} to optimize the performance difference between policies after and before policy updates using MACA advantages, that is, $L(\eta) = - \EE_\tau[R(\theta^{n+1}) - R(\theta^{n})]$, where $R(\theta)$ is the cumulative trajectory reward achieved by policy parameterized by $\theta$.




\section{EXPERIMENTS}\label{sec:experiment}
\looseness-1
In this section, we present empirical experiments and discuss the results.
We first describe the experimental setup.
Then we compare the performance of MACA and other state-of-the-art approaches.
We also present ablation studies to demonstrate the critical role of MACA's different components.

\begin{table*}[t]
\small
\centering
\caption{Mean evaluation win rate and standard deviation for different methods on SMAC v1\&v2 tasks. A win rate is marked in bold if it is within the critical region of the significance test.}
\label{tab:smac_win_rate}
\begin{adjustbox}{width=\textwidth,center}
\begin{tabular}{*{11}{c}}
\toprule
& Task & Task Type & MACA & MAPPO & IPPO & PPO-Mix & PPO-Sum & COMA & HAPPO & Steps \\
\midrule
\multirow{7}{*}{\rotatebox[origin=c]{90}{SMAC}}
& \textsl{25m} & homo. &$\bm{99.3}_{(0.1)}$ & $\bm{99.5}_{(0.3)}$ & $\bm{99.5}_{(0.5)}$ & $25.0_{(3.4)}$ & $68.0_{(14.9)}$ & $0.0_{(0.0)}$& $85.8_{(0.8)}$& 3e6 \\
& \textsl{5m\_vs\_6m} & homo. &$\bm{87.0}_{(2.0)}$ & $75.2_{(1.5)}$ & $78.0_{(0.9)}$ & $24.5_{(13.3)}$ & $67.0_{(3.2)}$ & $0.9_{(0.5)}$& $76.7_{(3.0)}$& 8e6 \\
& \textsl{8m\_vs\_9m} & homo. &$\bm{99.0}_{(0.6)}$ & $94.5_{(1.9)}$ & $95.0_{(1.4)}$ & $57.0_{(8.1)}$ & $66.0_{(13.2)}$ & $1.0_{(0.6)}$& $86.7_{(5.1)}$& 8e6 \\
& \textsl{10m\_vs\_11m} & homo. &$\bm{100.0}_{(0.0)}$ & $87.0_{(8.6)}$ & $98.8_{(0.7)}$ & $22.4_{(1.7)}$ & $69.0_{(16.7)}$ & $3.0_{(2.0)}$& $93.3_{(1.7)}$& 8e6 \\
& \textsl{3s5z} & hetero. &$\bm{99.2}_{(0.8)}$ & $96.5_{(0.9)}$ & $\bm{99.0}_{(1.0)}$ & $97.5_{(0.8)}$ & $0.5_{(1.1)}$ & $0.5_{(1.1)}$& $\bm{99.2}_{(0.8)}$& 8e6\\
\midrule
\multirow{6}{*}{\rotatebox[origin=c]{90}{SMACv2}}
& \textsl{protoss\_5\_vs\_5} & hetero. &$\bm{79.0}_{(3.4)}$  & $56.5_{(4.6)}$ & $54.2_{(2.7)}$ & $61.1_{(2.7)}$ & $32.5_{(6.8)}$ & $2.0_{(1.5)}$& $50.0_{(2.5)}$& 1e7 \\
& \textsl{terran\_5\_vs\_5} & hetero. &$\bm{74.4}_{(8.3)}$ & $50.5_{(2.7)}$ & $57.7_{(2.5)}$ & $57.0_{(3.8)}$ & $44.1_{(6.7)}$ & $7.0_{(2.2)}$& $55.8_{(3.6)}$& 1e7 \\
& \textsl{zerg\_5\_vs\_5} & hetero. &$\bm{63.4}_{(5.5)}$ & $42.3_{(1.4)}$ & $37.5_{(3.6)}$ & $41.5_{(4.4)}$ & $32.8_{(5.5)}$ & $4.5_{(0.9)}$& $42.5_{(1.4)}$& 1e7\\
& \textsl{protoss\_10\_vs\_10} & hetero. &$\bm{75.8}_{(3.9)}$ & $53.0_{(3.1)}$ & $38.0_{(4.0)}$ & $33.0_{(4.1)}$ & $26.0_{(3.4)}$ & $0.5_{(1.1)}$& $21.7_{(4.2)}$& 1e7 \\
& \textsl{terran\_10\_vs\_10} & hetero. &$\bm{75.0}_{(5.2)}$ & $40.0_{(5.2)}$ & $34.7_{(0.8)}$ & $37.6_{(3.3)}$ & $33.1_{(6.6)}$ & $2.5_{(1.4)}$& $17.5_{(3.8)}$& 1e7 \\
& \textsl{zerg\_10\_vs\_10} & hetero. &$\bm{62.9}_{(8.3)}$ & $39.8_{(3.0)}$ & $21.8_{(2.8)}$ & $32.0_{(2.2)}$ & $21.3_{(3.9)}$ & $1.0_{(0.6)}$& $17.5_{(5.2)}$& 1e7 \\
\bottomrule
\end{tabular}
\end{adjustbox}
\end{table*}

\subsection{Experimental Setup}
\paragraph{Baseline Methods}
\looseness-1
We compare MACA with SOTA multi-agent credit assignment approaches.
To promote fair comparison, we adopt all methods into the MAPPO learning framework.
We ensure the only difference among all methods is advantage function computation unless otherwise specified.
We use well-established implementations of MAPPO/IPPO/HAPPO \citep{zhong2023heterogeneous}, transformer encoder for MACA \citep{karpathy2022nanogpt,wen2022multi}, and
COMA/PPO-Mix/PPO-Sum \citep{papoudakis2020benchmarking}.
We include the following methods, whose corresponding advantage functions are summarized in \cref{tab:algo-overview}.

\textbf{MAPPO} optimizes all decentralized PPO \citep{schulman2017proximal} actors with a centralized critic.
MAPPO and IPPO are the SOTA on-policy MAPG algorithms in SMACv1 tasks \citep{yu2021surprising}.

\textbf{IPPO} (Independent PPO) directly adopts PPO that learns an individual critic for each actor.
IPPO performs competitively with MAPPO across different environments \citep{papoudakis2020benchmarking,yu2021surprising,de2020independent}.

\textbf{HAPPO} \citep{kuba2021trust} extends MAPPO by the sequential policy update scheme. 
Experimental results show competitive performance to MAPPO on SMACv1\&v2 \citep{zhong2023heterogeneous}.

\textbf{COMA} \citep{foerster2018counterfactual} performs counterfactual reasoning by marginalizing out the current agent's action in the baseline, as mentioned in previous sections.

\looseness-1
\textbf{PPO-Mix} is built upon the value decomposition method QMix \citep{rashid2018qmix} that learns the joint $Q$-value as a monotonic function of individual $Q$ values. PPO-Mix shares the same architecture as FACMAC \citep{peng2021facmac}, but the skeleton algorithm is the MAPPO rather than MADDPG to ensure a fair comparison.

\textbf{PPO-Sum} is identical to PPO-Mix, except the value decomposition method is VDN \citep{sunehag2017value}, which represents the joint value function as the sum of individual value functions. PPO-Sum is analogous to DOP \citep{wang2020off}, with MAPPO backbone instead of MADDPG for the same reason above.

\paragraph{Evaluation}
\looseness-1
We evaluate MACA on challenging benchmark StarCraft Multi-Agent Challenge (SMAC) \citep{samvelyan2019starcraft} (referred to as SMACv1 hereafter), and the recently proposed SMACv2 testbed \citep{ellis2022smacv2}.
We evaluate methods on five tasks from SMACv1, and six tasks from SMACv2,
covering a broad spectrum of agent types and task scenarios.
These tasks include both homogeneous and heterogeneous domains that ensure adequate coverage of task diversity and complexity, enabling a comprehensive empirical evaluation. 
We adopt the same hyperparameter settings from the original codebase and performed coarse finetuning, as detailed in 
\cref{app:hyperparam}.
We train each algorithm for 8M timesteps in SMACv1 tasks and 10M timesteps in SMACv2 tasks, with five random seeds per task.
During training, agents are evaluated every 160k timesteps with 40 independent runs.
We report the mean values and the standard deviation from evaluation win rates.
We also conduct experiments on three Multi-Agent Particle \citep{lowe2017multi,terry2021pettingzoo} (MPE) tasks following the same protocol.
We report the mean value and the standard deviation of evaluation episodic returns.

\begin{figure*}[!htp]
\centering
\includegraphics[width=\linewidth]{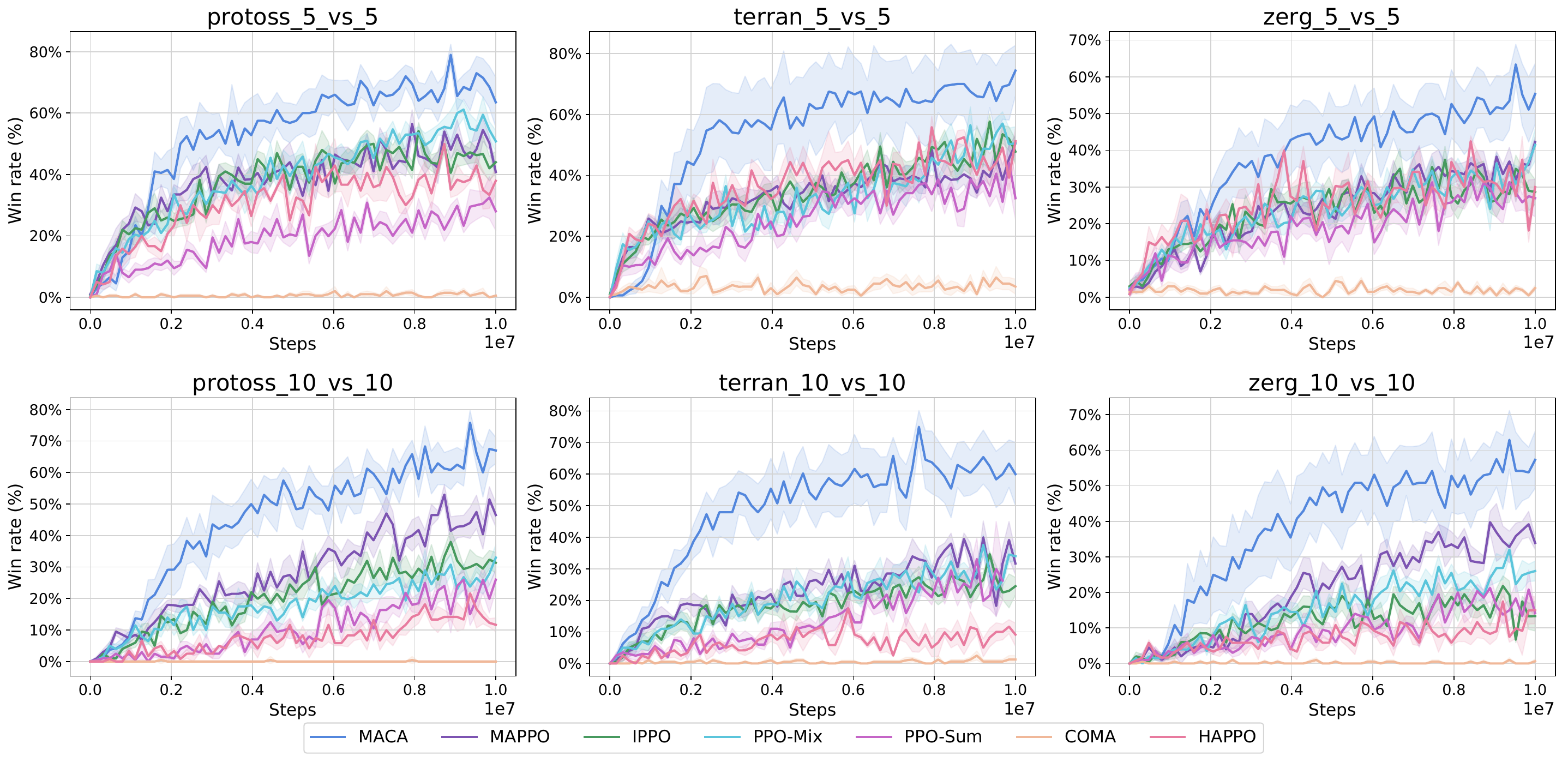}
\caption{Performance on the SMACv2 benchmark.}
\label{fig:smacv2_win_rate}
\end{figure*}

\subsection{Empirical Results}
In this section, we present and discuss evaluation results from experiments on SMACv1\&v2 testbeds.
\cref{tab:smac_win_rate} summarizes the evaluation win rates for all methods.
To determine the best performances, we perform a two-sample t-test \citep{snedecor1980statistical} with a significance level of 0.05 between the algorithm with the maximum mean and each of the other algorithms in every task.
Win rates that are not significantly different from the best performance are marked in bold.
\cref{fig:smacv2_win_rate} shows the learning progress of all methods in SMACv2 benchmarks.
We present learning curves in SMACv1 benchmarks, MPE results, and additional discussions in 
\cref{app:additional-results}. 


\looseness-1
Results in \cref{tab:smac_win_rate} and \cref{fig:smacv2_win_rate} show that MACA reaches stronger overall performance compared to the SOTA in SMACv1 benchmark, and superior performance in more challenging SMACv2 benchmarks.
Particularly in SMACv2 where MACA significantly outperforms all other methods, MACA gains higher sample efficiency.
MACA improves much faster as learning progresses, while other methods tend to reach convergence.
Such sample efficiency gain leads to higher final win rates.
Although MAPPO and PPO-Mix also demonstrate strong performance in two tasks, by our statistical test MACA still significantly outperforms both.

\looseness-1
In the SMACv1 benchmark, MACA achieves a superior overall performance.
Although MAPPO/IPPO/HAPPO performs competitively in the majority of tasks, they show inferior performance in certain tasks; thus MACA is significantly better than at least one of them in four out of five tasks.
Therefore, MACA demonstrates more \textit{robustness} than other methods in general.

\looseness-1
PPO-Mix generally performs well in SMACv2, but it occasionally fails in some SMACv1 tasks, e.g. in \textsl{25m}.
PPO-Sum shows large variances in SMACv1 tasks, and in general downperforms PPO-Mix.

COMA empirically performs poorly across all tasks in both SMACv1\&v2.
Such failure is consistent with prior works \citep{papoudakis2020benchmarking,kuba2021settling,wang2022shaq}.
\citet{wang2022shaq} argues COMA's poor performance may be due to relative overgeneralization, a common game theoretic pathology that the suboptimal actions are preferred \citep{wei2018multiagent}.

\begin{table*}[!htp]
\small
\centering
\caption{Mean evaluation win rate and standard deviation for MACA and its ablation variants.}
\label{tab:smac-ablation}
\begin{adjustbox}{width=\textwidth,center}
\begin{tabular}{*{9}{c}}
\toprule
Task & MACA & MACA-Dec & MACA-Jnt & MACA-Cor & MACA-Ind & MACA-NoCor & MACA-NoInd & MACA-NoJnt \\
\midrule
\textsl{5m\_vs\_6m} & $87.0_{(2.0)}$ & $84.0_{(2.3)}$ & $79.2_{(3.0)}$ & $70.6_{(1.9)}$ & $0.9_{(0.6)}$ & $75.0_{(1.8)}$ & $78.1_{(1.9)}$ & $70.0_{(4.7)}$ \\
\textsl{protoss\_10\_vs\_10} & $75.8_{(3.9)}$ & $64.7_{(5.1)}$ & $55.8_{(15.2)}$ & $59.2_{(11.7)}$ & $0.3_{(0.9)}$ & $52.5_{(6.0)}$ & $61.9_{(6.7)}$ & $65.0_{(7.6)}$ \\
\bottomrule
\end{tabular}
\end{adjustbox}
\end{table*}

\begin{figure}[!htb]
\centering
    \begin{subfigure}{0.49\linewidth}
        \centering
        \includegraphics[width=\linewidth]{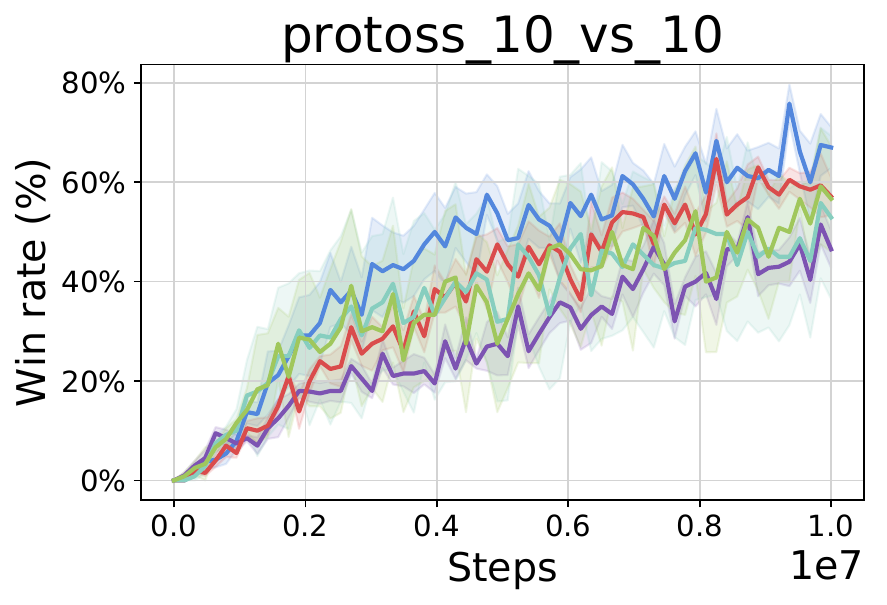}
    \end{subfigure}
    \begin{subfigure}{0.49\linewidth}
        \centering
        \includegraphics[width=\linewidth]{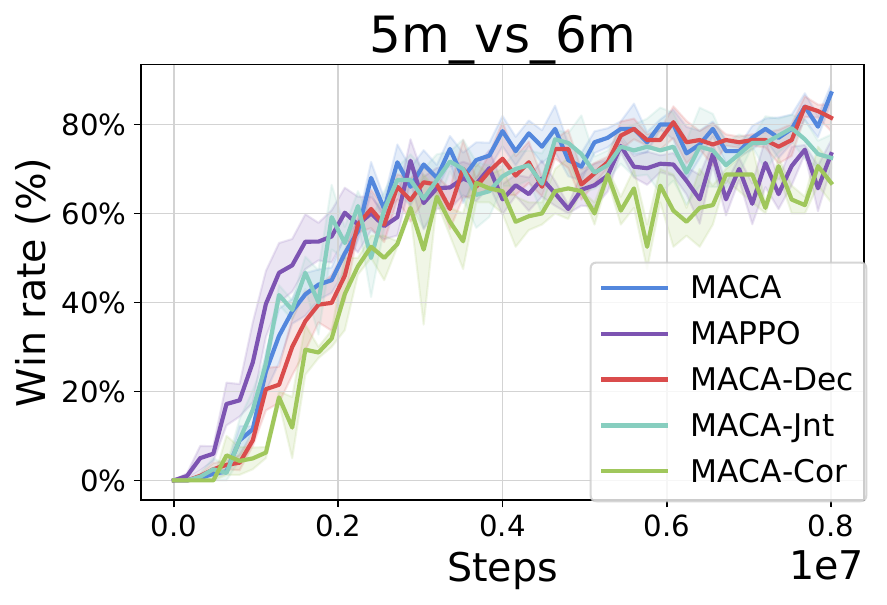}
    \end{subfigure}
    \begin{subfigure}{0.49\linewidth}
        \centering
        \includegraphics[width=\linewidth]{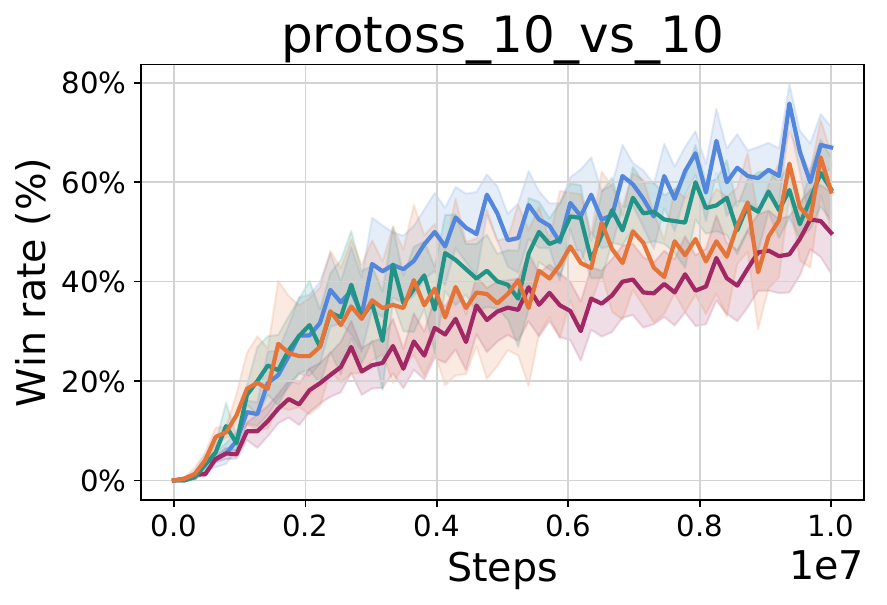}
    \end{subfigure}
    \begin{subfigure}{0.49\linewidth}
        \centering
        \includegraphics[width=\linewidth]{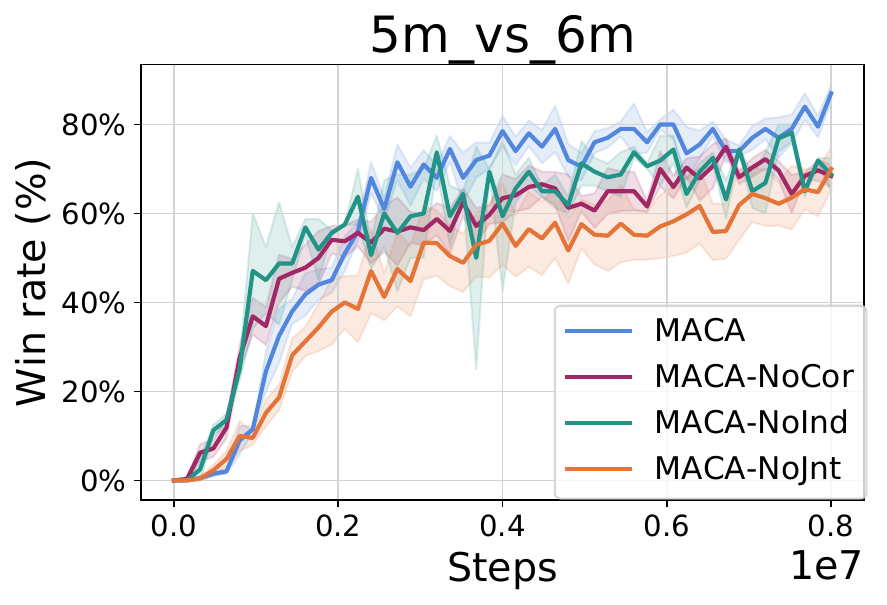}
    \end{subfigure}
\caption{Ablation performance.}
\label{fig:smac_win_rate_ablation}
\end{figure}

\vspace{-10pt}

\subsection{Ablations}\label{sec:experiment-ablation}
\looseness-1
In this section, we perform ablation experiments on MACA's key components.
Results are shown in \cref{tab:smac-ablation} and \cref{fig:smac_win_rate_ablation}.
We evaluate the following variants on \textsl{5m\_vs\_6m} and \textsl{protoss\_10\_vs\_10}.

\textit{MACA-Dec} is algorithmically equivalent to MACA, with the linear layer replaced with a transformer decoder \citep{vaswani2017attention} to perform value estimation.
A more detailed illustration is presented in 
\cref{fig:maca-architecture-decoder} in \cref{app:ablation}.
This variant aims to validate the expressivity of linear value estimation in MACA.
Results show that MACA is comparable with \textit{MACA-Dec} in both tasks, suggesting that the linear layer is capable of effectively estimating the value function.

\looseness-1
To ablate the importance of the MACA advantage, we remove different baseline components by enforcing their corresponding weighting coefficients to 0 and applying softmax to the remaining coefficients to obtain a valid categorical distribution.
\textit{MACA-Jnt} only uses the joint action set baseline $b^{\textsc{Jnt}}$.
Since this variant is equivalent to MAPPO with a transformer encoder-based critic, we attempt to verify (1) the effect of transformer encoder compared with MLP, and (2) the importance of multi-level advantage.
Results show that \textit{MACA-Jnt} reaches similar performance as MAPPO yet still downperforms MACA.
This demonstrates that improvement by only introducing a complex architecture is rather limited.
The MACA advantage plays a more important role in performance improvement.
\textit{MACA-Cor} only uses the \textit{CorrSet} baseline $b^{\textsc{Cor}}$.
This variant similarly downperforms MACA, highlighting the importance of the multi-level advantage.


\looseness-1
\textit{MACA-NoCor}, \textit{MACA-NoInd}, and \textit{MACA-NoJnt} remove the $b^{\textsc{Cor}}$, $b^{\textsc{Ind}}$, and $b^{\textsc{Jnt}}$, respectively,
seeking to evaluate individual components within the advantage baseline.
Results indicate that each advantage term is essential and contributes to performance improvement.

In addition, we conducted ablation experiments on \textit{MACA-Ind}, the variant that only keeps the $b^{Ind}$ baseline.
Similar to the equivalence between \textit{MACA-Jnt} and MAPPO, \textit{MACA-Ind} is essentially COMA, where the only difference is its attention-based critic. 
\textit{MACA-Ind} can also demonstrate the limited effect of only introducing the attention encoder, which aligns with conclusions from \textit{MACA-Jnt}. 
However, both COMA and \textit{MACA-Ind} empirically perform poorly, making the comparison not as informative as that between \textit{MACA-Jnt} and MAPPO.

\section{CONCLUSIONS}
\looseness-1
In this work, we tackle the credit assignment problem in cooperative MARL by considering different levels of credit assignment explicitly.
We formalize the per-level advantage that counterfactually deduces contributions over a specific level.
We propose MACA, an actor-critic method that constructs three different counterfactual advantage functions to respectively infer contributions from individual actions, joint actions, and actions taken by strongly correlated partners.
MACA leverages a transformer-based architecture to capture agents' correlations via the attention mechanism.
With multi-level contributions encoded by the combination of these advantages, MACA provides a generic formulation to address credit assignment challenges.
Empirical evaluations on challenging Starcraft benchmarks underscore MACA's superior performance, showcasing its efficacy in complex cooperative MARL scenarios.
Theoretical results provide support for MACA's strong performance.

In the future, it would be interesting to establish a connection between MACA and exploration methods via the advantage function, aiming to enhance performance across both regimes.

\textbf{Limitations.}
\label{subsec:limitations}
This work addresses the multi-agent credit assignment challenge in the cooperative setting, where agents receive collective rewards.
Other complex settings such as the mixed cooperative-competitive scenario could further complicate the problem, which is beyond the scope of this work.


\section*{Acknowledgments}
We appreciate the anonymous reviewers for their insightful comments and constructive suggestions that improved the quality of this manuscript.
We acknowledge the computational resources provided by Mila and the Digital Research Alliance of Canada.

\bibliography{maca}

\begin{thebibliography}{}

\bibitem[Abnar and Zuidema, 2020]{abnar2020quantifying}
Abnar, S. and Zuidema, W. (2020).
\newblock Quantifying attention flow in transformers.
\newblock {\em arXiv preprint arXiv:2005.00928}.

\bibitem[Albrecht et~al., 2023]{albrecht2023multi}
Albrecht, S.~V., Christianos, F., and Sch{\"a}fer, L. (2023).
\newblock Multi-agent reinforcement learning: Foundations and modern approaches.
\newblock {\em Massachusetts Institute of Technology: Cambridge, MA, USA}.

\bibitem[Ba et~al., 2016]{ba2016layer}
Ba, J.~L., Kiros, J.~R., and Hinton, G.~E. (2016).
\newblock Layer normalization.
\newblock {\em arXiv preprint arXiv:1607.06450}.

\bibitem[Baker et~al., 2019]{baker2019emergent}
Baker, B., Kanitscheider, I., Markov, T., Wu, Y., Powell, G., McGrew, B., and Mordatch, I. (2019).
\newblock Emergent tool use from multi-agent autocurricula.
\newblock {\em arXiv preprint arXiv:1909.07528}.

\bibitem[Bernstein et~al., 2002]{bernstein2002complexity}
Bernstein, D.~S., Givan, R., Immerman, N., and Zilberstein, S. (2002).
\newblock The complexity of decentralized control of markov decision processes.
\newblock {\em Mathematics of operations research}, 27(4):819--840.

\bibitem[Chen et~al., 2021]{chen2021decision}
Chen, L., Lu, K., Rajeswaran, A., Lee, K., Grover, A., Laskin, M., Abbeel, P., Srinivas, A., and Mordatch, I. (2021).
\newblock Decision transformer: Reinforcement learning via sequence modeling.
\newblock {\em Advances in neural information processing systems}, 34:15084--15097.

\bibitem[Chung et~al., 2021]{chung2021beyond}
Chung, W., Thomas, V., Machado, M.~C., and Le~Roux, N. (2021).
\newblock Beyond variance reduction: Understanding the true impact of baselines on policy optimization.
\newblock In {\em International Conference on Machine Learning}, pages 1999--2009. PMLR.

\bibitem[Dao et~al., 2022]{dao2022flashattention}
Dao, T., Fu, D., Ermon, S., Rudra, A., and R{\'e}, C. (2022).
\newblock Flashattention: Fast and memory-efficient exact attention with io-awareness.
\newblock {\em Advances in neural information processing systems}, 35:16344--16359.

\bibitem[de~Witt et~al., 2020]{de2020independent}
de~Witt, C.~S., Gupta, T., Makoviichuk, D., Makoviychuk, V., Torr, P.~H., Sun, M., and Whiteson, S. (2020).
\newblock Is independent learning all you need in the starcraft multi-agent challenge?
\newblock {\em arXiv preprint arXiv:2011.09533}.

\bibitem[Ellis et~al., 2022]{ellis2022smacv2}
Ellis, B., Cook, J., Moalla, S., Samvelyan, M., Sun, M., Mahajan, A., Foerster, J.~N., and Whiteson, S. (2022).
\newblock Smacv2: An improved benchmark for cooperative multi-agent reinforcement learning.
\newblock {\em arXiv preprint arXiv:2212.07489}.

\bibitem[Foerster et~al., 2018]{foerster2018counterfactual}
Foerster, J., Farquhar, G., Afouras, T., Nardelli, N., and Whiteson, S. (2018).
\newblock Counterfactual multi-agent policy gradients.
\newblock In {\em Proceedings of the AAAI conference on artificial intelligence}, volume~32.

\bibitem[Hansen et~al., 2019]{hansen2019pycma}
Hansen, N., Akimoto, Y., and Baudis, P. (2019).
\newblock {CMA-ES/pycma} on {G}ithub.
\newblock Zenodo, DOI:10.5281/zenodo.2559634.

\bibitem[H{\"u}ttenrauch et~al., 2017]{huttenrauch2017guided}
H{\"u}ttenrauch, M., {\v{S}}o{\v{s}}i{\'c}, A., and Neumann, G. (2017).
\newblock Guided deep reinforcement learning for swarm systems.
\newblock {\em arXiv preprint arXiv:1709.06011}.

\bibitem[Jensen, 1906]{jensen1906fonctions}
Jensen, J. L. W.~V. (1906).
\newblock Sur les fonctions convexes et les in{\'e}galit{\'e}s entre les valeurs moyennes.
\newblock {\em Acta mathematica}, 30(1):175--193.

\bibitem[Karpathy, 2023]{karpathy2022nanogpt}
Karpathy, A. (2023).
\newblock nanogpt.
\newblock \url{https://github.com/karpathy/nanoGPT}.

\bibitem[Konda and Tsitsiklis, 1999]{konda1999actor}
Konda, V. and Tsitsiklis, J. (1999).
\newblock Actor-critic algorithms.
\newblock {\em Advances in neural information processing systems}, 12.

\bibitem[Kuba et~al., 2021a]{kuba2021trust}
Kuba, J.~G., Chen, R., Wen, M., Wen, Y., Sun, F., Wang, J., and Yang, Y. (2021a).
\newblock Trust region policy optimisation in multi-agent reinforcement learning.
\newblock {\em arXiv preprint arXiv:2109.11251}.

\bibitem[Kuba et~al., 2021b]{kuba2021settling}
Kuba, J.~G., Wen, M., Meng, L., Zhang, H., Mguni, D., Wang, J., Yang, Y., et~al. (2021b).
\newblock Settling the variance of multi-agent policy gradients.
\newblock {\em Advances in Neural Information Processing Systems}, 34:13458--13470.

\bibitem[Kumar et~al., 2020]{kumar2020problems}
Kumar, I.~E., Venkatasubramanian, S., Scheidegger, C., and Friedler, S. (2020).
\newblock Problems with shapley-value-based explanations as feature importance measures.
\newblock In {\em International Conference on Machine Learning}, pages 5491--5500. PMLR.

\bibitem[Li et~al., 2021]{li2021shapley}
Li, J., Kuang, K., Wang, B., Liu, F., Chen, L., Wu, F., and Xiao, J. (2021).
\newblock Shapley counterfactual credits for multi-agent reinforcement learning.
\newblock In {\em Proceedings of the 27th ACM SIGKDD Conference on Knowledge Discovery \& Data Mining}, pages 934--942.

\bibitem[Li et~al., 2022]{li2022difference}
Li, Y., Xie, G., and Lu, Z. (2022).
\newblock Difference advantage estimation for multi-agent policy gradients.
\newblock In {\em International Conference on Machine Learning}, pages 13066--13085. PMLR.

\bibitem[Liu et~al., 2022]{liu2022stateful}
Liu, D., Shah, V., Boussif, O., Meo, C., Goyal, A., Shu, T., Mozer, M., Heess, N., and Bengio, Y. (2022).
\newblock Stateful active facilitator: Coordination and environmental heterogeneity in cooperative multi-agent reinforcement learning.
\newblock {\em arXiv preprint arXiv:2210.03022}.

\bibitem[Lowe et~al., 2017]{lowe2017multi}
Lowe, R., Wu, Y.~I., Tamar, A., Harb, J., Pieter~Abbeel, O., and Mordatch, I. (2017).
\newblock Multi-agent actor-critic for mixed cooperative-competitive environments.
\newblock {\em Advances in neural information processing systems}, 30.

\bibitem[Meng et~al., 2021]{meng2021offline}
Meng, L., Wen, M., Yang, Y., Le, C., Li, X., Zhang, W., Wen, Y., Zhang, H., Wang, J., and Xu, B. (2021).
\newblock Offline pre-trained multi-agent decision transformer: One big sequence model tackles all smac tasks.
\newblock {\em arXiv preprint arXiv:2112.02845}.

\bibitem[Nayak et~al., 2023]{nayak2023scalable}
Nayak, S., Choi, K., Ding, W., Dolan, S., Gopalakrishnan, K., and Balakrishnan, H. (2023).
\newblock Scalable multi-agent reinforcement learning through intelligent information aggregation.
\newblock In {\em International Conference on Machine Learning}, pages 25817--25833. PMLR.

\bibitem[Oliehoek and Amato, 2016]{oliehoek2016concise}
Oliehoek, F.~A. and Amato, C. (2016).
\newblock {\em A concise introduction to decentralized POMDPs}.
\newblock Springer.

\bibitem[Papoudakis et~al., 2020]{papoudakis2020benchmarking}
Papoudakis, G., Christianos, F., Sch{\"a}fer, L., and Albrecht, S.~V. (2020).
\newblock Benchmarking multi-agent deep reinforcement learning algorithms in cooperative tasks.
\newblock {\em arXiv preprint arXiv:2006.07869}.

\bibitem[Peng et~al., 2021]{peng2021facmac}
Peng, B., Rashid, T., Schroeder~de Witt, C., Kamienny, P.-A., Torr, P., B{\"o}hmer, W., and Whiteson, S. (2021).
\newblock Facmac: Factored multi-agent centralised policy gradients.
\newblock {\em Advances in Neural Information Processing Systems}, 34:12208--12221.

\bibitem[Rashid et~al., 2018]{rashid2018qmix}
Rashid, T., Samvelyan, M., Schroeder, C., Farquhar, G., Foerster, J., and Whiteson, S. (2018).
\newblock Qmix: Monotonic value function factorisation for deep multi-agent reinforcement learning.
\newblock In {\em International conference on machine learning}, pages 4295--4304. PMLR.

\bibitem[Richerson et~al., 2016]{richerson2016cultural}
Richerson, P., Baldini, R., Bell, A.~V., Demps, K., Frost, K., Hillis, V., Mathew, S., Newton, E.~K., Naar, N., Newson, L., et~al. (2016).
\newblock Cultural group selection plays an essential role in explaining human cooperation: A sketch of the evidence.
\newblock {\em Behavioral and Brain Sciences}, 39:e30.

\bibitem[Roesch et~al., 2020]{roesch2020smart}
Roesch, M., Linder, C., Zimmermann, R., Rudolf, A., Hohmann, A., and Reinhart, G. (2020).
\newblock Smart grid for industry using multi-agent reinforcement learning.
\newblock {\em Applied Sciences}, 10(19):6900.

\bibitem[Samvelyan et~al., 2019]{samvelyan2019starcraft}
Samvelyan, M., Rashid, T., De~Witt, C.~S., Farquhar, G., Nardelli, N., Rudner, T.~G., Hung, C.-M., Torr, P.~H., Foerster, J., and Whiteson, S. (2019).
\newblock The starcraft multi-agent challenge.
\newblock {\em arXiv preprint arXiv:1902.04043}.

\bibitem[Schulman et~al., 2017]{schulman2017proximal}
Schulman, J., Wolski, F., Dhariwal, P., Radford, A., and Klimov, O. (2017).
\newblock Proximal policy optimization algorithms.
\newblock {\em arXiv preprint arXiv:1707.06347}.

\bibitem[Seraj et~al., 2022]{seraj2022learning}
Seraj, E., Wang, Z., Paleja, R., Patel, A., and Gombolay, M. (2022).
\newblock Learning efficient diverse communication for cooperative heterogeneous teaming.
\newblock Technical report, Sandia National Lab.(SNL-NM), Albuquerque, NM (United States).

\bibitem[Shalev-Shwartz et~al., 2016]{shalev2016safe}
Shalev-Shwartz, S., Shammah, S., and Shashua, A. (2016).
\newblock Safe, multi-agent, reinforcement learning for autonomous driving.
\newblock {\em arXiv preprint arXiv:1610.03295}.

\bibitem[Shapley, 1953]{shapley1953value}
Shapley, L.~S. (1953).
\newblock A value for n-person games.
\newblock {\em Contribution to the Theory of Games}, 2.

\bibitem[Snedecor and Cochran, 1980]{snedecor1980statistical}
Snedecor, G.~W. and Cochran, W.~G. (1980).
\newblock Statistical methods. iowa.
\newblock {\em Iowa State University Press. Starkstein, SE, \& Robinson, RG (1989). Affective disorders and cerebral vascular disease. The British Journal of Psychiatry}, 154:170--182.

\bibitem[Son et~al., 2019]{son2019qtran}
Son, K., Kim, D., Kang, W.~J., Hostallero, D.~E., and Yi, Y. (2019).
\newblock Qtran: Learning to factorize with transformation for cooperative multi-agent reinforcement learning.
\newblock In {\em International conference on machine learning}, pages 5887--5896. PMLR.

\bibitem[Sunehag et~al., 2017]{sunehag2017value}
Sunehag, P., Lever, G., Gruslys, A., Czarnecki, W.~M., Zambaldi, V., Jaderberg, M., Lanctot, M., Sonnerat, N., Leibo, J.~Z., Tuyls, K., et~al. (2017).
\newblock Value-decomposition networks for cooperative multi-agent learning.
\newblock {\em arXiv preprint arXiv:1706.05296}.

\bibitem[Sutton et~al., 1999]{sutton1999policy}
Sutton, R.~S., McAllester, D., Singh, S., and Mansour, Y. (1999).
\newblock Policy gradient methods for reinforcement learning with function approximation.
\newblock {\em Advances in neural information processing systems}, 12.

\bibitem[Terry et~al., 2021]{terry2021pettingzoo}
Terry, J., Black, B., Grammel, N., Jayakumar, M., Hari, A., Sullivan, R., Santos, L.~S., Dieffendahl, C., Horsch, C., Perez-Vicente, R., et~al. (2021).
\newblock Pettingzoo: Gym for multi-agent reinforcement learning.
\newblock {\em Advances in Neural Information Processing Systems}, 34:15032--15043.

\bibitem[Tumer and Agogino, 2007]{tumer2007distributed}
Tumer, K. and Agogino, A. (2007).
\newblock Distributed agent-based air traffic flow management.
\newblock In {\em Proceedings of the 6th international joint conference on Autonomous agents and multiagent systems}, pages 1--8.

\bibitem[Vaswani et~al., 2017]{vaswani2017attention}
Vaswani, A., Shazeer, N., Parmar, N., Uszkoreit, J., Jones, L., Gomez, A.~N., Kaiser, {\L}., and Polosukhin, I. (2017).
\newblock Attention is all you need.
\newblock {\em Advances in neural information processing systems}, 30.

\bibitem[Wang et~al., 2020a]{wang2020qplex}
Wang, J., Ren, Z., Liu, T., Yu, Y., and Zhang, C. (2020a).
\newblock Qplex: Duplex dueling multi-agent q-learning.
\newblock {\em arXiv preprint arXiv:2008.01062}.

\bibitem[Wang et~al., 2022]{wang2022shaq}
Wang, J., Zhang, Y., Gu, Y., and Kim, T.-K. (2022).
\newblock Shaq: Incorporating shapley value theory into multi-agent q-learning.
\newblock {\em Advances in Neural Information Processing Systems}, 35:5941--5954.

\bibitem[Wang et~al., 2020b]{wang2020off}
Wang, Y., Han, B., Wang, T., Dong, H., and Zhang, C. (2020b).
\newblock Off-policy multi-agent decomposed policy gradients.
\newblock {\em arXiv preprint arXiv:2007.12322}.

\bibitem[Weaver and Tao, 2013]{weaver2013optimal}
Weaver, L. and Tao, N. (2013).
\newblock The optimal reward baseline for gradient-based reinforcement learning.
\newblock {\em arXiv preprint arXiv:1301.2315}.

\bibitem[Wei et~al., 2018]{wei2018multiagent}
Wei, E., Wicke, D., Freelan, D., and Luke, S. (2018).
\newblock Multiagent soft q-learning.
\newblock {\em arXiv preprint arXiv:1804.09817}.

\bibitem[Wen et~al., 2022]{wen2022multi}
Wen, M., Kuba, J.~G., Lin, R., Zhang, W., Wen, Y., Wang, J., and Yang, Y. (2022).
\newblock Multi-agent reinforcement learning is a sequence modeling problem.
\newblock {\em arXiv preprint arXiv:2205.14953}.

\bibitem[Wierstra and Schmidhuber, 2007]{wierstra2007policy}
Wierstra, D. and Schmidhuber, J. (2007).
\newblock Policy gradient critics.
\newblock In {\em European Conference on Machine Learning}, pages 466--477. Springer.

\bibitem[Wolpert and Tumer, 2001]{wolpert2001optimal}
Wolpert, D.~H. and Tumer, K. (2001).
\newblock Optimal payoff functions for members of collectives.
\newblock {\em Advances in Complex Systems}, 4(02n03):265--279.

\bibitem[Yu et~al., 2021]{yu2021surprising}
Yu, C., Velu, A., Vinitsky, E., Wang, Y., Bayen, A., and Wu, Y. (2021).
\newblock The surprising effectiveness of ppo in cooperative, multi-agent games.
\newblock {\em arXiv preprint arXiv:2103.01955}.

\bibitem[Zhang et~al., 2022]{zhang2022relational}
Zhang, F., Liu, B., Wang, K., Tan, V., Yang, Z., and Wang, Z. (2022).
\newblock Relational reasoning via set transformers: Provable efficiency and applications to marl.
\newblock {\em Advances in Neural Information Processing Systems}, 35:35825--35838.

\bibitem[Zhang et~al., 2018]{zhang2018fully}
Zhang, K., Yang, Z., Liu, H., Zhang, T., and Basar, T. (2018).
\newblock Fully decentralized multi-agent reinforcement learning with networked agents.
\newblock In {\em International Conference on Machine Learning}, pages 5872--5881. PMLR.

\bibitem[Zhong et~al., 2023]{zhong2023heterogeneous}
Zhong, Y., Kuba, J.~G., Hu, S., Ji, J., and Yang, Y. (2023).
\newblock Heterogeneous-agent reinforcement learning.
\newblock {\em arXiv preprint arXiv:2304.09870}.

\bibitem[Zhou et~al., 2020]{zhou2020learning}
Zhou, M., Liu, Z., Sui, P., Li, Y., and Chung, Y.~Y. (2020).
\newblock Learning implicit credit assignment for cooperative multi-agent reinforcement learning.
\newblock {\em Advances in neural information processing systems}, 33:11853--11864.

\end{thebibliography}

\section*{Checklist}



 \begin{enumerate}

 \item For all models and algorithms presented, check if you include:
 \begin{enumerate}
   \item A clear description of the mathematical setting, assumptions, algorithm, and/or model. [Yes/No/Not Applicable] Yes. 
   \item An analysis of the properties and complexity (time, space, sample size) of any algorithm. [Yes/No/Not Applicable] Yes. 
   \item (Optional) Anonymized source code, with specification of all dependencies, including external libraries. [Yes/No/Not Applicable] Not Applicable.
 \end{enumerate}

 \item For any theoretical claim, check if you include:
 \begin{enumerate}
   \item Statements of the full set of assumptions of all theoretical results. [Yes/No/Not Applicable] Yes.
   \item Complete proofs of all theoretical results. [Yes/No/Not Applicable] Yes.
   \item Clear explanations of any assumptions. [Yes/No/Not Applicable] Yes.
 \end{enumerate}

 \item For all figures and tables that present empirical results, check if you include:
 \begin{enumerate}
   \item The code, data, and instructions needed to reproduce the main experimental results (either in the supplemental material or as a URL). [Yes/No/Not Applicable] Yes.
   \item All the training details (e.g., data splits, hyperparameters, how they were chosen). [Yes/No/Not Applicable] Yes.
         \item A clear definition of the specific measure or statistics and error bars (e.g., with respect to the random seed after running experiments multiple times). [Yes/No/Not Applicable] Yes.
         \item A description of the computing infrastructure used. (e.g., type of GPUs, internal cluster, or cloud provider). [Yes/No/Not Applicable] Yes.
 \end{enumerate}

 \item If you are using existing assets (e.g., code, data, models) or curating/releasing new assets, check if you include:
 \begin{enumerate}
   \item Citations of the creator If your work uses existing assets. [Yes/No/Not Applicable] Yes.
   \item The license information of the assets, if applicable. [Yes/No/Not Applicable] Yes.
   \item New assets either in the supplemental material or as a URL, if applicable. [Yes/No/Not Applicable] Yes.
   \item Information about consent from data providers/curators. [Yes/No/Not Applicable] Yes.
   \item Discussion of sensible content if applicable, e.g., personally identifiable information or offensive content. [Yes/No/Not Applicable] Not Applicable.
 \end{enumerate}

 \item If you used crowdsourcing or conducted research with human subjects, check if you include:
 \begin{enumerate}
   \item The full text of instructions given to participants and screenshots. [Yes/No/Not Applicable] Not Applicable.
   \item Descriptions of potential participant risks, with links to Institutional Review Board (IRB) approvals if applicable. [Yes/No/Not Applicable] Not Applicable.
   \item The estimated hourly wage paid to participants and the total amount spent on participant compensation. [Yes/No/Not Applicable] Not Applicable.
 \end{enumerate}

 \end{enumerate}

\appendix
\onecolumn
\aistatstitle{Multi-level Advantage Credit Assignment\\for Cooperative Multi-Agent Reinforcement Learning: \\
Supplementary Materials}
\section{Theoretical Results}\label{app:theoretical-results}

\begin{lemma}\label{thm:baseline-unbiasedness}
Action-independent baseline functions do not affect the bias of the policy gradient estimate in expectation, i.e.,
\begin{align*}
    g_b &= \EE_{s \sim d^\pi, a \sim \pi} \left[ b_i \nabla_{\theta_i} \log \pi_i(a_i|s) \right] =0
\end{align*}
for any agent $i \in \agentset$ if $b_i$ does not depend on $a_i$.

\end{lemma}

\begin{proof}
We assume continuous action space. Proof for the discrete case is analogous.
\begin{align*}
    g_b &= \EE_{s \sim d^\pi, a \sim \pi} \left[ b_i \nabla_{\theta_i} \log \pi_i(a_i|s) \right] \\
        &= \EE_{s \sim d^\pi, a_{-i} \sim \pi_{-i}} \left[ b_i \EE_{a_{i} \sim \pi_{i}} \left[ \nabla_{\theta_i} \log \pi_i(a_i|s) \right] \right].
\end{align*}
We have
\begin{align*}
    & \EE_{a_{i} \sim \pi_{i}} \left[ \nabla_{\theta_i} \log \pi_i(a_i|s) \right] = \int_{\actionspace_i} \pi_i(a_i|s) \, \nabla_{\theta_i} \log \pi_i(a_i|s) \, da_i \\
    &= \int_{\actionspace_i} \nabla_{\theta_i} \pi_i(a_i|s) \, da_i = \nabla_{\theta_i} \int_{\actionspace_i} \pi_i(a_i|s) \, da_i = \nabla_{\theta_i} 1 = 0
\end{align*}
which concludes the proof.

\end{proof}

\begin{theorem}[Minimum-variance baseline]\label{thm:min-var-baseline}
Adapted from single-agent results by \citet{chung2021beyond}. \\
The minimum-variance baseline for the MAPG estimator is
\begin{align*}
    b^*_i(s, a_{-\levelset_i}) &= \frac{\EE_{a_{\levelset_i}} \left[ Q(s,a) || \nabla_{\theta_i} \log \pi_i(a_i|s) ||^2 \right]}{\EE_{a_{\levelset_i}} \left[ || \nabla_{\theta_i} \log \pi_i(a_i|s) ||^2 \right]}
\end{align*}
if the baseline is conditioned on the state $s$ and actions by $-\levelset_i$, where $a_{\levelset_i} = \{a_j \sim \pi_j: j \in \levelset_i \}$, $\{i\} \subset \levelset_i \subset \agentset$, and $-\levelset_i = \agentset \setminus \levelset_i$.

\end{theorem}

\begin{proof}
Assuming access to the Q-value for each state-action pair $Q(s, a)=Q^\pi(s, a)$, the gradient is given in 
\cref{eq:ma-pg} 
as $\nabla_{\theta_i} J(\theta) = \EE_{s \sim d^\pi, a \sim \pi} \left[ Q(s, a) \nabla_{\theta_i} \log \pi_i(a_i|s) \right] = \EE_{s \sim d^\pi, a \sim \pi} \left[ g_i \right]$, where the
estimator is $g_i = (Q(s, a) - b_i(s, a_{-\levelset_i})) \nabla_{\theta_i} \log \pi_i(a_i|s)$.
We have
\begin{align*}
    \Var_{a \sim \pi}\left[g_i\right] &= \underbrace{\Var_{a_{-\levelset_i}} \left[ \EE_{a_{\levelset_i}}  \left[g_i\right] \right]}_{\text{variance from given actions}} + \underbrace{\EE_{a_{-\levelset_i}} \left[\Var_{a_{\levelset_i}} \left[ g_i \right] \right]}_{\text{variance from unknown actions}}
\end{align*}
We now derive the baseline that minimizes the second term.
\begin{align*}
    \Var_{a_{\levelset_i}}(g_i) &= \EE_{a_{\levelset_i}} \left[ \|g_i \|^2 \right] - \| \EE_{a_{\levelset_i}} \left[ g_i \right] \|^2 \\
    &= \EE_{a_{\levelset_i}} \left[ \|g_i \|^2 \right] - \| \EE_{a_{\levelset_i}} \left[ Q(s, a) \nabla_{\theta_i} \log \pi_i(a_i|s) \right] \|^2
\end{align*}

\pagebreak 
The equality holds by proof of \cref{thm:baseline-unbiasedness}. 
Then we only need to consider the first item
\begin{align*} 
    \frac{\partial}{\partial b} \EE_{a_{\levelset_i}} \left[ \|g_i\|^2 \right] =& \frac{\partial}{\partial b} \EE_{a_{\levelset_i}} \left[ \| Q(s, a) \nabla \log \pi_i(a_i|s) \|^2 - 2 \cdot Q(s, a) b_i(s,a_{-\levelset_i}) \| \nabla \log \pi_i(a_i|s) \|^2 \right. \\
    & \left. +b_i(s,a_{-\levelset_i})^2 \| \nabla \log \pi_i(a_i|s) \|^2 \right] \\ 
    =& 2 \left( b_i(s,a_{-\levelset_i}) \cdot \EE_{a_{\levelset_i}} \left[ \| \nabla \log \pi_i(a_i|s) \|^2 \right] - \EE_{a_{\levelset_i}} \left[ Q(s, a) \| \nabla \log \pi_i(a_i|s) \|^2 \right] \right)
\end{align*}

The minimum variance can be obtained by setting the gradient above to 0, i.e.,
\begin{align*}
    b^*_i(s, a_{-\levelset_i}) &= \frac{\EE_{a_{\levelset_i}} \left[ Q(s,a) || \nabla_{\theta_i} \log \pi_i(a_i|s) ||^2 \right]}{\EE_{a_{\levelset_i}} \left[ || \nabla_{\theta_i} \log \pi_i(a_i|s) ||^2 \right]}.
\end{align*}

\end{proof}

\begin{lemma}\label{thm:min-var-vs-value}
Adapted from single-agent results by \citet{chung2021beyond}. \\
The minimum-variance baseline $b^*_i$ from \cref{thm:min-var-baseline} and the baseline $b_i = \EE{a_{\levelset_i}} [Q(s,a)]$ satisfies
\begin{align*}
    b_i &= b^*_i - \frac{\Cov_{a_{\levelset_i}} \left( Q(s,a), || \nabla_{\theta_i} \log \pi_i(a_i|s) ||^2 \right)}{\EE_{a_{\levelset_i}} \left[ || \nabla_{\theta_i} \log \pi_i(a_i|s) ||^2 \right]}.
\end{align*}

\end{lemma}

\begin{proof}
We have
\begin{align*}
    b^*_i(s, a_{-\levelset_i}) &= \frac{\EE_{a_{\levelset_i}} \left[ Q(s,a) || \nabla_{\theta_i} \log \pi_i(a_i|s) ||^2 \right]}{\EE_{a_{\levelset_i}} \left[ || \nabla_{\theta_i} \log \pi_i(a_i|s) ||^2 \right]} \\
    &= \frac{\EE_{a_{\levelset_i}} \left[ Q(s,a) || \nabla_{\theta_i} \log \pi_i(a_i|s) ||^2 \right]}{\EE_{a_{\levelset_i}} \left[ || \nabla_{\theta_i} \log \pi_i(a_i|s) ||^2 \right]} - b_i + b_i \\
    &= \frac{\EE_{a_{\levelset_i}} \left[ Q(s,a) || \nabla_{\theta_i} \log \pi_i(a_i|s) ||^2 \right] - \EE_{a_{\levelset_i}} \left[Q(s,a)\right] \EE_{a_{\levelset_i}}\left[ || \nabla_{\theta_i} \log \pi_i(a_i|s) ||^2 \right]}{\EE_{a_{\levelset_i}} \left[ || \nabla_{\theta_i} \log \pi_i(a_i|s) ||^2 \right]} + b_i \\
    &= \frac{\Cov_{a_{\levelset_i}} \left( Q(s,a), || \nabla_{\theta_i} \log \pi_i(a_i|s) ||^2 \right)}{\EE_{a_{\levelset_i}} \left[ || \nabla_{\theta_i} \log \pi_i(a_i|s) ||^2 \right]} + b_i
\end{align*}

\end{proof}

\begin{lemma}\label{thm:suboptimality}
Adapted from \citet{foerster2018counterfactual}.\\
For an actor-critic algorithm with a compatible TD(1) critic following a policy gradient
\begin{align*}
    g^k &= \EE_{s \sim d^\pi, a \sim \pi} \left[ \sum_i \nabla_{\theta^k} \log \pi_i(a_i|s) \left( Q(s,a)-b_i \right) \right]
\end{align*}
at each iteration $k$, where the baseline function $b_i$ does not depend on action $a_i$,
\begin{align*}
\lim \inf_k ||\nabla J|| = 0 \quad w.p. \; 1.
\end{align*}
\end{lemma}

\begin{proof}
By \cref{thm:baseline-unbiasedness}, the per-agent baseline does not change the expected gradient, thereby not affecting the convergence.
Hence we have the policy gradient
\begin{align*}
    g &= \EE_{s \sim d^\pi, a \sim \pi} \left[ \sum_i \nabla_{\theta} \log \pi_i(a_i|s) \left( Q(s,a)-b_i \right) \right] \\
    &= \EE_{s \sim d^\pi, a \sim \pi} \left[ \sum_i \nabla_{\theta} \log \pi_i(a_i|s) Q(s,a) \right] \\
    &= \EE_{s \sim d^\pi, a \sim \pi} \left[ \nabla_{\theta} \log \prod_i \pi_i(a_i|s) Q(s,a) \right] \\
    &= \EE_{s \sim d^\pi, a \sim \pi} \left[ \nabla_{\theta} \log \pi(a|s) Q(s,a) \right]
\end{align*}
yielding the standard single-agent actor-critic policy gradient.

\citet{konda1999actor} prove that an actor-critic following this gradient converges to a local maximum of the expected return $J(\theta)$, given that:
\begin{enumerate}
    \item the policy $\pi$ is differentiable,
    \item the update timescales for $Q$ and $\pi$ are sufficiently slow, and that $\pi$ is updated sufficiently slower than $Q$, and
    \item $Q$ uses a representation compatible with $\pi$,
\end{enumerate}
amongst several further assumptions. 
The policy parameterization (i.e., the joint-action agent is decomposed into independent actors) is immaterial to convergence, as long as it remains differentiable. 
Note that a centralized critic with access to the global state is essential for this proof to hold.

\end{proof}

\section{Experiment Details}\label{app:additional-results}

\subsection{Experimental Setup}
SMACv1 contains a diverse set of battle scenarios in which a team of ally units aims to defeat the enemy team.
SMACv2 extends SMACv1 with procedurally generated scenarios, with team compositions and positions spawned according to different probabilities,
which introduces significantly more stochasticity that requires agents' generalization to unseen settings.
We follow the default reward setting, which returns positive rewards for damage dealt to the enemy, killing each enemy unit, and winning the battle.
Tasks in both benchmarks embody complicated multi-agent credit assignment challenges, due to diverse cooperative behaviours involved in defeating the opponent.
\cref{tab:smac_desc} provides an overview of SMAC v1\&v2 task scenarios.
\begin{table*}[t]
\small
\centering
\caption{SMAC v1\&v2 task scenarios with unit and task types.}
\label{tab:smac_desc}
\begin{adjustbox}{width=0.8\textwidth,center}
\begin{tabular}{*{5}{c}}
\toprule
& Task  & Ally Unit Types & Task Type  \\
\midrule
\multirow{7}{*}{\rotatebox[origin=c]{90}{SMAC}}
& \textsl{25m}  & Marine & homogeneous \\
& \textsl{5m\_vs\_6m}  & Marine & homogeneous \\
& \textsl{8m\_vs\_9m}  & Marine & homogeneous \\
& \textsl{10m\_vs\_11m}  & Marine & homogeneous \\
& \textsl{3s5z}  & Stalker,Zealot & \textbf{heterogeneous} \\
\midrule
\multirow{6}{*}{\rotatebox[origin=c]{90}{SMACv2}}
& \textsl{protoss\_5\_vs\_5}  & Stalker, Zealot, Colossus & \textbf{heterogeneous} \\
& \textsl{terran\_5\_vs\_5}  & Marine, Marauder, Medivac & \textbf{heterogeneous} \\
& \textsl{zerg\_5\_vs\_5}  & Zergling, Hydralisk, Baneling & \textbf{heterogeneous} \\
& \textsl{protoss\_10\_vs\_10}  & Stalker, Zealot, Colossus & \textbf{heterogeneous} \\
& \textsl{terran\_10\_vs\_10}  & Marine, Marauder, Medivac & \textbf{heterogeneous} \\
& \textsl{zerg\_10\_vs\_10}  & Zergling, Hydralisk, Baneling & \textbf{heterogeneous} \\
\bottomrule
\end{tabular}
\end{adjustbox}
\end{table*}

\subsection{Computational Requirements}\label{app:compute}
It is important to highlight that MACA utilizes the attention encoder module of the transformer, rather than the entire transformer architecture. 
Besides self-attention's internal computations (i.e., QKV projection, convex combination), other operations within MACA's critic, including policy distribution computation, value estimation, and weighted sum, are all linear operations.
MACA hence enjoys reasonable time and space complexity compared to other baseline methods.
In practice, MACA has memory and runtime on the same scale as other methods despite its transformer module, as shown in 
\cref{tab:algo-num-param-runtime}.
Moreover, MACA's attention-based critic only affects parameter updates, while the majority of time consumed lies in agent-environment interactions, where only the actor networks are involved.
As modern deep learning frameworks optimize the efficiency of attention (e.g. FlashAttention \citep{dao2022flashattention}), the use of attention would not bottleneck MACA's scalability. 

We conduct all our experiments entirely on CPUs in compute clusters with multiple nodes.
All experiments were run on a single CPU core.
The main CPU model types are AMD Rome 7532 @ 2.40 GHz 256M cache L3 and AMD Rome 7502 @ 2.50 GHz 128M cache L3.
Each SMACv1/SMACv2/MPE job respectively took 24/48/12 CPU hours.
The total number of CPU hours spent (including hyperparameter search) is 33,708.

\subsection{Additional Results and Discussions}
\begin{table*}[!htb]
\small
\centering
\caption{Mean evaluation episodic returns and standard deviation for different methods on MPE tasks.}
\label{tab:mpe_return}
\begin{adjustbox}{width=\textwidth,center}
\begin{tabular}{*{10}{c}}
\toprule
& Task & MACA & MAPPO & IPPO & PPO-Mix & PPO-Sum & COMA & HAPPO & Steps \\
\midrule
\multirow{3}{*}{\rotatebox[origin=c]{90}{MPE}}
& \textsl{Spread} & $-61.95_{(1.44)}$ & $-67.96_{(4.22)}$ & $-67.01_{(2.40)}$ & $-106.56_{(0.92)}$ & $-84.29_{(1.14)}$ & $-109.09_{(3.61)}$ & $-72.25_{(2.05)}$& 8e6 \\
& \textsl{Reference} & $-12.04_{(0.45)}$ & $-12.19_{(0.43)}$ & $-29.63_{(1.16)}$ & $-27.84_{(2.64)}$ & $-34.40_{(0.79)}$ & $-39.71_{(1.76)}$ & $-12.44_{(0.43)}$& 8e6 \\
& \textsl{Speaker Listener} & $-9.17_{(0.53)}$ & $-9.10_{(0.53)}$ & $-9.41_{(0.58)}$ & $-21.26_{(2.52)}$ & $-26.06_{(1.55)}$ & $-31.80_{(2.94)}$ & $-9.06_{(0.51)}$& 8e6 \\
\bottomrule
\end{tabular}
\end{adjustbox}
\end{table*}

\begin{figure*}[!htp]
\centering
\includegraphics[width=1.0\linewidth]{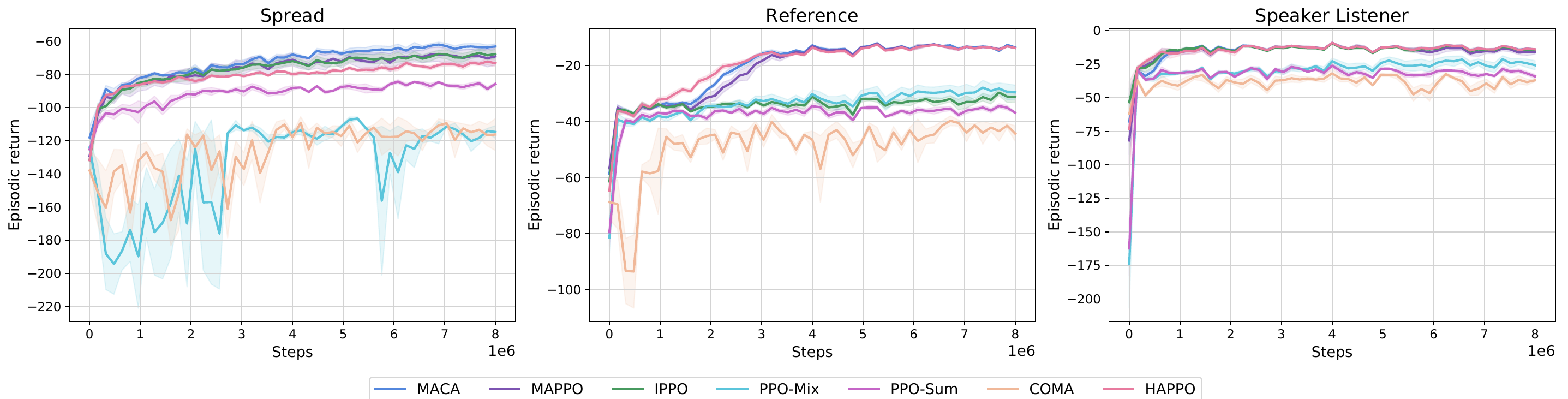}
\caption{Performance on the MPE benchmark.}
\label{fig:mpe_return}
\end{figure*}
\begin{figure*}[!htp]
\centering
\includegraphics[width=1.0\linewidth]{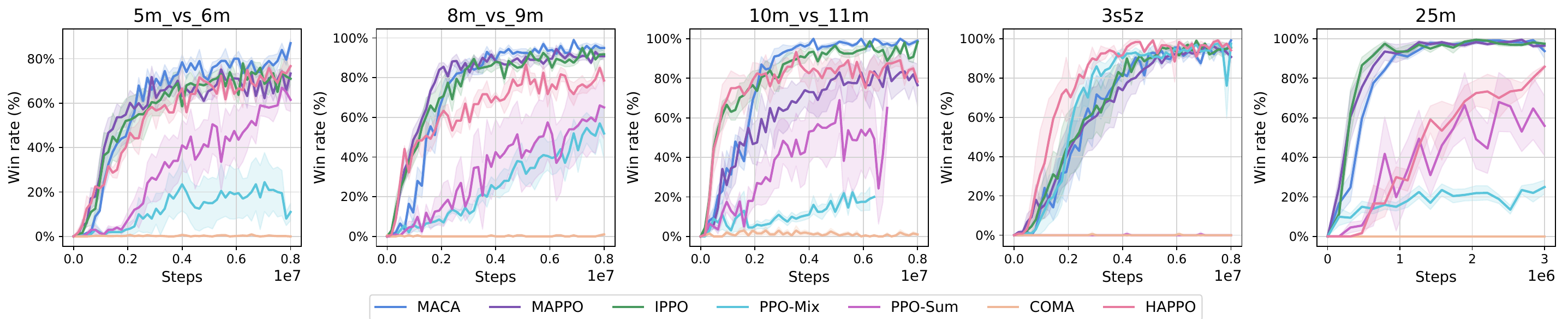}
\caption{Performance on the SMAC benchmark.}
\label{fig:smac_win_rate}
\end{figure*}



\looseness-1
On MPE tasks, multiple baseline methods exhibit strong performance. 
Despite continuous action space, MPE tasks are easy because they involve similar cooperation types and simple state-action spaces, leading to limited evaluation challenges.
MACA marginally outperforms other methods unlike in the more challenging SMAC environments.

\begin{table}[!htb] 
\centering
\small
\caption{Number of parameters in the critic model and experiment runtime. All values are reported in the case of \textsl{5m\_vs\_6m}. MACA has memory and runtime requirements on the same scale as other methods.}
\label{tab:algo-num-param-runtime}
\begin{tabular}{*{3}{c}}
\toprule
Algorithm & Number of parameters & Runtime (hours) \\
\midrule
MAPPO & 165,747 & 15.1 \\
IPPO & 149,497 & 14.8 \\
COMA & 174,966 & 18.6 \\
HAPPO & 165,747 & 15.9 \\
PPO-Mix & 361,349 & 18.5 \\
PPO-Sum & 134,148 & 17.4 \\
MACA & 239,674 & 16.7 \\
\bottomrule
\end{tabular}
\end{table}

\subsection{Visualization}
\begin{figure}[!htp] 
\centering
\resizebox{0.8\textwidth}{!}{
    \begin{subfigure}{0.45\linewidth}
        \centering
        \includegraphics[width=\linewidth]{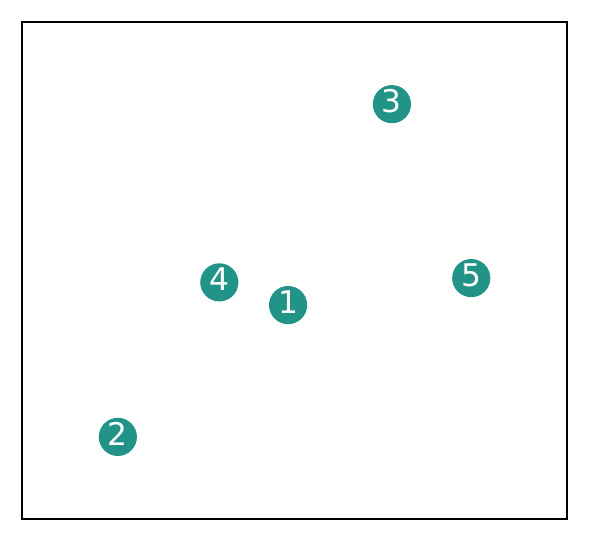}
    \end{subfigure}
    \begin{subfigure}{0.53\linewidth}
        \centering
        \includegraphics[width=\linewidth]{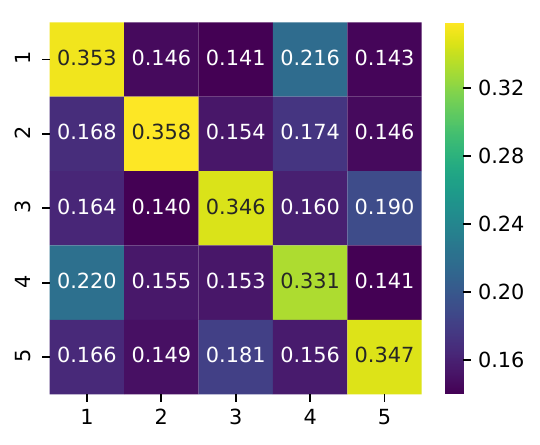}
    \end{subfigure}
}
\caption{Example visualized ally agents' coordinates (left) and corresponding attention weight (right) in \textsl{5m\_vs\_6m}. The attention matrix shows that agents 1 and 4 have high scores, which aligns with the task map where these agents are allied together.}
\label{fig:attn_visualize}
\end{figure}
We present an exemplary visualization of ally coordinates and corresponding attention weights in \cref{fig:attn_visualize},
which shows that the attention map reflects meaningful inter-agent correlations.

\subsection{Ablation Details}\label{app:ablation}
\begin{figure*}[!htp]
\centering
\includegraphics[width=0.9\linewidth]{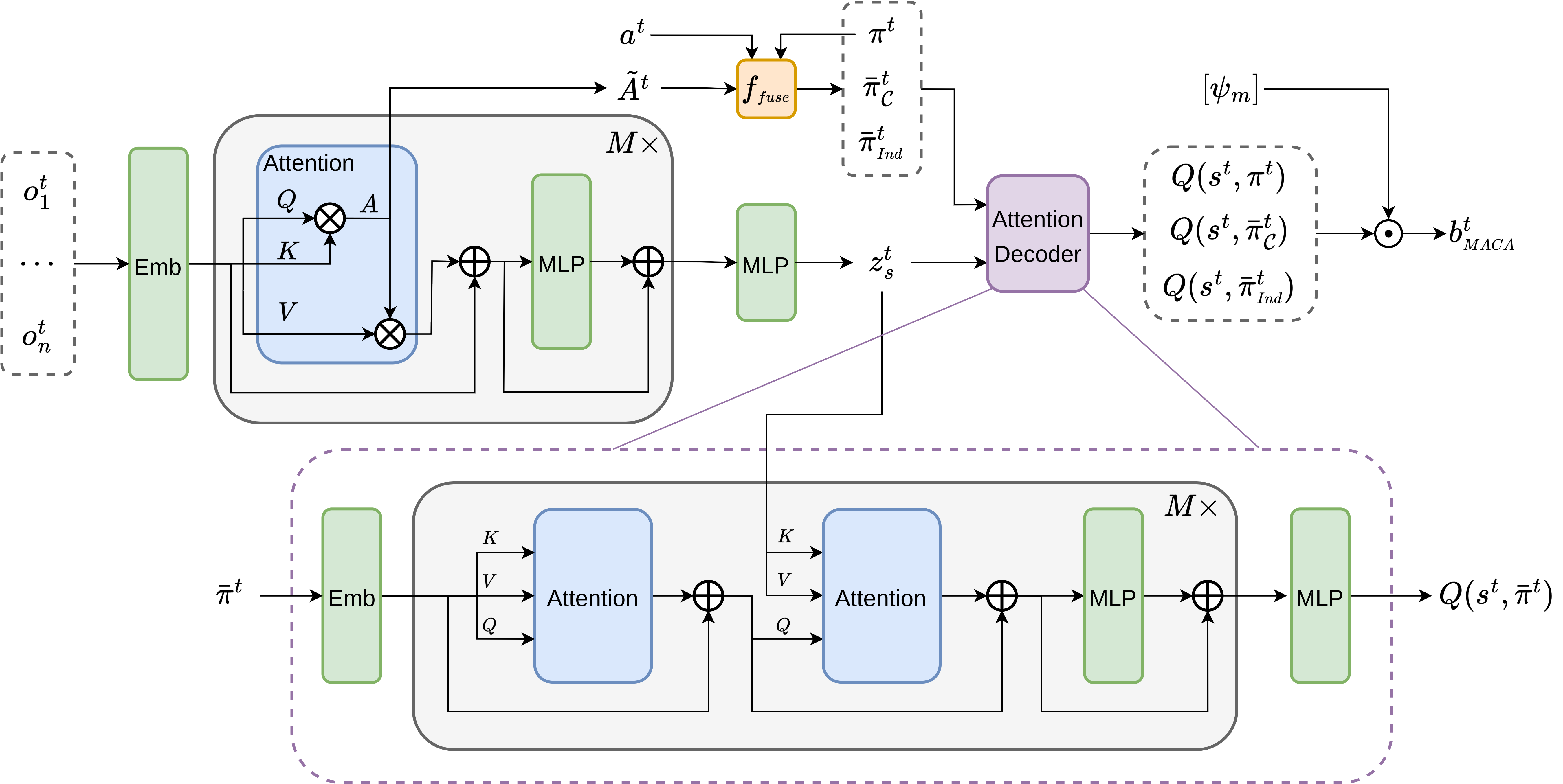}
\caption{MACA critic architecture with transformer decoder.}
\label{fig:maca-architecture-decoder}
\end{figure*}
\textit{MACA-Dec} replaces the linear layer with a transformer decoder to perform value estimation.
The architecture is illustrated in \cref{fig:maca-architecture-decoder}.
\section{Hyperparameters}\label{app:hyperparam}
\begin{table*}[!htp]
\small
\centering
\caption{Common hyperparameters across tasks.}
\label{tab:hyperparam-common}
\begin{tabular}{*{2}{c}}
\toprule
Hyperparameter & Value \\
\midrule
state type & EP \\
lr & 0.0005 \\
actor hidden sizes & [64,64,64] \\
critic hidden sizes & [128,128,128] \\
actor num mini batch & 1 \\
critic num mini batch & 1 \\
gamma &  0.99 \\
entropy coef & 0.01 \\
use valuenorm  & True  \\
use linear lr decay & False \\
use proper time limits &  True \\
activation  & ReLU  \\
use feature normalization  & True \\
initialization method  & orthogonal  \\
gain  & 0.01 \\
use naive recurrent policy  & False  \\
use recurrent policy & True \\
num GRU layers  & 1 \\
data chunk length  & 10  \\
optim eps  & 1e-5 \\
weight decay  & 0  \\
std x coef  & 1 \\
std y coef  & 0.5  \\
use clipped value loss  & True \\
value loss coef  & 1  \\
use max grad norm  & True \\
max grad norm  & 10.0  \\
use GAE  & True \\
GAE lambda &  0.95 \\ 
use huber loss  & True \\
use policy active masks  & True  \\
huber delta  & 10.0 \\
action aggregation  & prod \\ 
share param & True \\

\bottomrule
\end{tabular}
\end{table*}

\begin{table*}[!htp]
\small
\centering
\caption{Hyperparameters for MAPPO.}
\label{tab:hyperparam-mappo}
\begin{tabular}{*{3}{c}}
\toprule
Task & ppo/critic epoch & clip param \\
\midrule
SMAC & 10 & 0.1 \\
SMACv2 & 10 & 0.1 \\
MPE & 10 & 0.1 \\

\bottomrule
\end{tabular}
\end{table*}

\begin{table*}[!htp]
\small
\centering
\caption{Hyperparameters for IPPO.}
\label{tab:hyperparam-ippo}
\begin{tabular}{*{3}{c}}
\toprule
Task & ppo/critic epoch & clip param \\
\midrule
SMAC & 10 & 0.1 \\
SMACv2 & 10 & 0.1 \\
MPE & 10 & 0.1 \\

\bottomrule
\end{tabular}
\end{table*}

\begin{table*}[!htp]
\small
\centering
\caption{Common hyperparameters for MACA across tasks.}
\label{tab:hyperparam-common-maca}
\begin{tabular}{*{2}{c}}
\toprule
Hyperparameter & Value \\
\midrule
critic hidden sizes & [64,64,64] \\
n encode layer & 1 \\
n head & 1 \\
n embd & 64 \\
zs dim & 256 \\
bias & True \\
active fn & gelu \\
weight decay & 0.01 \\
betas & [0.9, 0.95] \\
weight init & TFixup \\
warmup epochs & 10 \\
v value loss coef & 1.0 \\
q value loss coef & 0.5 \\
att sigma $\sigma$ & $1.0/n$ \\

\bottomrule
\end{tabular}
\end{table*}

\begin{table*}[!htp]
\small
\centering
\caption{Hyperparameters for MACA.}
\label{tab:hyperparam-maca}
\begin{tabular}{*{3}{c}}
\toprule
Task & ppo/critic epoch & clip param \\
\midrule
SMAC & 10 & 0.05 \\
SMACv2 & 10 & 0.1  \\
MPE & 10 & 0.1 \\

\bottomrule
\end{tabular}
\end{table*}

\begin{table*}[!htp]
\small
\centering
\caption{Hyperparameters for PPO-Mix.}
\label{tab:hyperparam-ppo-mix}
\begin{tabular}{*{3}{c}}
\toprule
Task & ppo/critic epoch & clip param \\
\midrule
SMAC & 10 & 0.05 \\
SMACv2 & 10 & 0.05 \\
MPE & 10 & 0.05 \\

\bottomrule
\end{tabular}
\end{table*}

\begin{table*}[!htp]
\small
\centering
\caption{Hyperparameters for PPO-Sum.}
\label{tab:hyperparam-ppo-sum}
\begin{tabular}{*{3}{c}}
\toprule
Task & ppo/critic epoch & clip param \\
\midrule
SMAC & 10 & 0.05 \\
SMACv2 & 10 & 0.05 \\
MPE & 10 & 0.05 \\

\bottomrule
\end{tabular}
\end{table*}

\begin{table*}[!htp]
\small
\centering
\caption{Hyperparameters for COMA.}
\label{tab:hyperparam-coma}
\begin{tabular}{*{3}{c}}
\toprule
Task & ppo/critic epoch & clip param \\
\midrule
SMAC & 5 & 0.05 \\
SMACv2 & 5 & 0.05 \\
MPE & 5 & 0.05 \\

\bottomrule
\end{tabular}
\end{table*}

\begin{table*}[!htp]
\small
\centering
\caption{Hyperparameters for HAPPO.}
\label{tab:hyperparam-happo}
\begin{tabular}{*{3}{c}}
\toprule
Task & ppo/critic epoch & clip param \\
\midrule
SMAC & 10 & 0.1 \\
SMACv2 & 10 & 0.1 \\
MPE & 10 & 0.1 \\

\bottomrule
\end{tabular}
\end{table*}

\begin{table*}[!htp]
\small
\centering
\caption{Hyperparameter sweep for MACA. Results demonstrate MACA's robustness across ppo/critic epoch and clip hyperparameter settings in multiple representative environments.}
\label{tab:hyperparam-sweep-epoch-clip}
\begin{tabular}{*{5}{c}}
\toprule
epoch/clip & \textsl{25m} & \textsl{5m\_vs\_6m} & \textsl{protoss\_5\_vs\_5} & \textsl{protoss\_10\_vs\_10} \\
\midrule
10/0.05 & 99.3 (0.1) & 87.0 (2.0) & 76.1 (3.5) & 70.1 (2.6) \\
10/0.075 & 99.3 (0.1) & 85.5 (2.8) & 77.5 (4.8) & 72.4 (4.1) \\ 
10/0.1 & 99.5 (0.1) & 82.5 (4.1) & 79.0 (3.4) & 75.8 (3.9) \\
10/0.125 & 99.1 (0.0) & 80.6 (3.7) & 78.2 (2.7) & 74.6 (5.2) \\ 
5/0.05 & 99.3 (0.3) & 83.4 (2.3) & 70.6 (6.4) & 67.9 (6.0) \\ 
5/0.1 & 99.1 (0.1) & 80.2 (2.5) & 76.3 (5.6) & 71.5 (7.6) \\
\bottomrule
\end{tabular}
\end{table*}

\begin{table*}[!htp]
\small
\centering
\caption{Hyperparameter sweep for MACA. Results demonstrate MACA's robustness across thresholding hyperparameter settings in multiple representative environments.}
\label{tab:hyperparam-sweep-threshold}
\begin{tabular}{*{6}{c}}
\toprule
threshold $\sigma$ & 0.9/n & 0.95/n & 1.0/n & 1.05/n & 1.1/n \\
\midrule
\textsl{5m\_vs\_6m} & 86.5 (1.5) & 87.4 (1.2) & 87.0 (2.0) & 86.6 (0.9) & 85.7 (3.3) \\
\textsl{protoss\_10\_vs\_10} & 74.2 (4.4) & 75.3 (4.6) & 75.8 (3.9) & 74.7 (3.6) & 74.5 (5.0) \\
\bottomrule
\end{tabular}
\end{table*}

\end{document}